\documentclass{article}

\usepackage[preprint]{neurips_2022}

\usepackage[utf8]{inputenc} % allow utf-8 input
\usepackage[T1]{fontenc}    % use 8-bit T1 fonts
\usepackage{hyperref}       % hyperlinks
\usepackage{url}            % simple URL typesetting
\usepackage{booktabs}       % professional-quality tables
\usepackage{amsfonts}       % blackboard math symbols
\usepackage{nicefrac}       % compact symbols for 1/2, etc.
\usepackage{microtype}      % microtypography
\usepackage{xcolor}         % colors
\usepackage{caption}
\usepackage{amsmath}
\usepackage{amssymb}
\usepackage{amsthm}
\DeclareMathOperator*{\argmax}{arg\,max}
\DeclareMathOperator*{\argmin}{arg\,min}
\usepackage{multicol}
\usepackage{multirow}
\usepackage{tabularx}
\usepackage{diagbox}
\usepackage{subfigure}
\usepackage{array}
\usepackage{cleveref}
\usepackage{xcolor}
\usepackage{tikz}
\usepackage{verbatim}
\usepackage{listings}
\definecolor{mygreen}{RGB}{28,172,0} % color values Red, Green, Blue
\definecolor{mylilas}{RGB}{170,55,241}

\usepackage{matlab-prettifier}
\usepackage{color}
\usepackage[shortlabels]{enumitem}

\usepackage{algorithm}
\usepackage[noend]{algpseudocode}
\makeatletter
\def\BState{\State\hskip-\ALG@thistlm}
\makeatother
\newcommand{\E}{\mathbb{E}}

\newcommand{\tp}{\mathsf{T}}

\newcommand{\N}{\mathbb{N}}
\newcommand{\R}{\mathbb{R}}

\newtheorem{theorem}{Theorem}
\newtheorem{lemma}{Lemma}
\newtheorem{proposition}{Proposition}
\newtheorem{definition}{Definition}
\newtheorem{assumption}{Assumption}
\title{Sampling Attacks on Meta Reinforcement Learning: A Minimax Formulation and Complexity Analysis}

\author{%
  Tao Li, Haozhe Lei, and  Quanyan Zhu \\
  Department of Electrical and Computer Engineering\\
  New York University\\
  NY 11201 \\
  \texttt{\{taoli, hl4155, qz494\}@nyu.edu} \\
}

\begin{document}

\maketitle

\begin{abstract}
Meta reinforcement learning (meta RL), as a combination of meta-learning ideas and reinforcement learning (RL), enables the agent to adapt to different tasks using a few samples. However, this sampling-based adaptation also makes meta  RL vulnerable to adversarial attacks. By manipulating the reward feedback from sampling processes in meta RL, an attacker can mislead the agent into building wrong knowledge from training experience, which deteriorates the agent's performance when dealing with different tasks after adaptation. This paper provides a game-theoretical underpinning for understanding this type of security risk. In particular, we formally define the sampling attack model as a Stackelberg game between the attacker and the agent, which yields a minimax formulation. It leads to two online attack schemes: Intermittent Attack and Persistent Attack, which enable the attacker to learn an optimal sampling attack, defined by an $\epsilon$-first-order stationary point, within $\mathcal{O}(\epsilon^{-2})$ iterations. These attack schemes freeride the learning progress concurrently without extra interactions with the environment.  By corroborating the convergence results with numerical experiments, we observe that a minor effort of the attacker can significantly deteriorate the learning performance, and the minimax approach can also help robustify the meta RL algorithms. Code and demos are available at \verb|https://github.com/NYU-LARX/Attack_MetaRL|
\end{abstract}

\section{Introduction}
Meta Reinforcement Learning (meta RL), as a combination of meta learning and reinforcement learning, intends to learn a decision-making model that can quickly adapt to new tasks. Compared with learning from scratch, meta RL enables the agent to leverage previous training data, requiring fewer samples when dealing with new tasks.  The successes of meta RL \citep{16learnRL, abbeel16rl2, finn2017model, fallah_sgmrl} can be explained from a feature learning standpoint: meta RL aims at building an internal representation of the policy (meta policy) that is broadly suitable for a collection of similar tasks. This internal representation is learned by either recurrent neural networks \citep{16learnRL,abbeel16rl2} or stochastic optimization \citep{finn2017model,fallah_sgmrl}, when agents are exposed to the collection of tasks via a unique sampling process. In meta RL, in addtion to producing sample trajectories within a single task, the sampling process also randomly  samples environments for learning a common representation shared across these tasks. 

However, this unique sampling process also makes meta RL vulnerable when facing adversarial attacks. By manipulating sample data collected in the meta training phase, the attacker can mislead the agent into learning a wrong meta policy that adapts poorly to individual tasks from the collection. For example, in Model-Agnostic Meta Reinforcement Learning (MAMRL) \citep{finn2017model}, a unique sampling process is required for optimizing the meta policy. As shown by our numerical results, slightly perturbing the reward feedback from this sampling, the attacker can significantly degrade the performance of meta RL agents under different tasks.    

\textbf{Contributions} This paper aims to  provide a theoretical understanding of adversarial attacks on sampling processes of meta RL, which is referred to as sampling attacks.  Due to its mathematical clarity, we choose MAMRL \citep{finn2017model} as the departure point of our discussion. Yet, the idea of sampling attacks also applies to other meta RL formulations \citep{16learnRL,abbeel16rl2} (see \Cref{app:relevance}). 

Our contributions are three-fold. 1) We characterize the sampling attack as a Stackelberg game \citep{stackelberg} between the attacker and the agent (the learner), which yields a minimax formulation. 2) Based on this formulation, two online black-box training-time attack schemes are proposed: Intermittent Attack and Persistent Attack, which enable the attacker to learn optimal sampling attacks concurrently with the meta learning without extra interactions with the environment. 3) The optimality of sampling attacks is defined by $\epsilon$-first-order stationarity, and our nonasymptotic analysis shows that the proposed attack schemes can achieve the optimality within $\mathcal{O}(\epsilon^{-2})$ iterations.

To the best of our knowledge, this work is among the first endeavor to investigate online training-time adversarial attacks on sampling processes of meta RL. In addition, our analysis of nonconvex-nonconcave minimax problem can be extended to other RL problems sharing the same nature of nonconvexity. 

\section{Related Works}\label{sec:related}This section reviews recent progress in the area of adversarial attacks on RL.   We position our work using the following categorizations and comment on the differences between existing approaches and the proposed one in this paper, highlighting our contribution to the security study of RL and meta RL. 

\textbf{Testing-time and Training-time Attacks} 
Unlike testing-time attacks, where the attacker intends to degrade the performance of a learned and deployed policy, training-time attacks aim to enforce a target policy in favor of the attacker by manipulating the learning process.

Our work falls within the class of training-time attack. Even though our attack methods are based on reward manipulations, which have also been explored in the literature \citep{ma2019policy,yunhan19,zhang2020adaptive,zhang21npg},  we emphasize that previous works mainly target Q-learning algorithm \citep{ma2019policy,yunhan19,zhang2020adaptive} or neural policy gradient \citep{zhang21npg}, which relies on $Q$ value estimation. For these value-based algorithms, the influence of the reward manipulation on the value function is more straightforward than in meta RL, where the adaptation in the meta training complicates the analysis. The optimality results regarding previous attacks do not apply to meta RL, and our paper's theoretical analysis of the optimality of the reward manipulation is more challenging.

\textbf{White-box and Black-box Attacks} For both testing-time and training-time attacks on RL agents, existing works have investigated both white-box \citep{ma2019policy,zhang2020adaptive,yunhan19} and black-box settings \citep{xu21transfer,russoACC21state_perturb}. For white-box attacks, the attacker needs to fully know the learner's learning algorithm and the policy model. Some attacks additionally require the knowledge of  the environment, e.g., transition dynamics \citep{xu21transfer}. 

In contrast, the proposed attacks, as a black-box attack\footnote{The fault line between white-box and black-box may vary in different contexts, We here follow the notion used in \citep{xu21transfer} and treat our proposed attack as a black-box attack.}, do not require prior knowledge of the environment setup or learning algorithms. With access to the policy model, the attacker can adjust its attack by gradient updates, regardless of learning algorithms or the environment parameters. 

Compared with previous black-box attacks \citep{xu21transfer,russoACC21state_perturb}, the proposed attacks do not require interactions with the training environment, e.g., the simulator. All it needs are the sample trajectories rolled out by the learner during the training process. In other words, the attacker considered in this paper is a free rider who does not actively collect information regarding the agent's learning algorithm or the environment.    

\textbf{Offline and Online Attacks} 
Previous works mainly consider offline training time reward manipulations \citep{ma2019policy,yunhan19}, where the attacker has full knowledge of the training dataset. The attacker makes one decision on reward manipulation and then provides the poisoned data to RL agent. In this work, online attacks are studied, where the attacker manipulates the reward feedback sequentially, adaptive to the learner's learning process. 

\section{Preliminaries}\label{sec:pre}
 Let $\{\mathcal{T}_i\}_{i=1}^T$ be the set of finite-horizon Markov Decision Processes (MDP) representing $T$ different tasks/environments. Each task $\mathcal{T}_i$ is given by a tuple $\left\langle \mathcal{S},\mathcal{A},P_i, r_i, \rho_i\right\rangle$, where $\mathcal{S}$ and $\mathcal{A}$ denotes the the space of states and actions, respectively. For each $\mathcal{T}_i$, $P_i: \mathcal{S}\times \mathcal{A}\rightarrow\mathcal{B}(\mathcal{S})$ denotes the transition kernel that maps a state-action pair to  $\mathcal{B}(\mathcal{S})$, a Borel probability measure over the state space. Finally, the agent receives a reward $r_i(s,a)$ when taking action $a$ at state $s$, and the initial state is drawn from the initial distribution $\rho_i$. For the sake of simplicity, we assume that all tasks share the same state and action spaces, as well as the same horizon length $H$, and our analysis applies to cases where different tasks have different MDP formulations \citep{fallah_sgmrl}.  
 
 To facilitate our discussion, the following notations are introduced.  A trajectory from an MDP $\mathcal{T}_i$ is defined as $\tau({\mathcal{T}_i})=(s_1,a_1,\ldots, s_H,a_H)$, where $H$ denotes the horizon length. A batch of trajectories $\tau({\mathcal{T}_i})$ under the same MDP is considered to be a training data set $\mathcal{D}_i$. Given a trajectory $\tau_i$ (a shorthand for $\tau(\mathcal{T}_i)$), the total reward received over this trajectory is $R(\tau_i):=\sum_{t=1}^H r_i(s_t,a_t)$.
 
 Suppose that the agent's policy in $\mathcal{T}_i$, $\pi(\cdot|\cdot;\theta):\mathcal{S}\rightarrow\mathcal{B}(\mathcal{A})$, is parameterized with $\theta\in  R^d$, and $\pi(a|s;\theta)$ is the probability of choosing action $a$ at state $s$. Then, given the transition dynamics $P_i(s'|s,a)$ and the initial distribution $\rho_i(s)$, the probability that such a trajectory $\tau_i$ occurs is $q_i(\tau;\theta):=\rho_i(s_1)\prod_{t=1}^H\pi(a_t|s_t;\theta)\prod_{t=1}^{H-1}P_i(s_{t+1}|s_t,a_t).$  Therefore, the expected cumulative rewards under the policy $\pi(\cdot|\cdot;\theta)$ in the task $\mathcal{T}_i$ is $J_i(\theta)= \E_{\tau_i\sim q_i(\cdot;\theta)}[R(\tau_i)]$. RL algorithms, such as  policy gradient \citep{sutton_PG},  seek to find a $\theta^*$ that maximizes  $  J_i(\theta)$.

The idea of policy gradient method is to apply gradient descent with respect to the objective function $J_i(\theta)$. Following the policy gradient theorem \citep{sutton_PG}, we obtain $\nabla J_i(\theta)=\mathbb{E}_{\tau\sim q_i(\cdot;\theta)}[g_i(\tau;\theta)], g_i(\tau;\theta)=\sum_{h=1}^H\nabla_\theta\log \pi_i(a_h|s_h;\theta)R_i^h(\tau)$ where $R^h_i(\tau)=\sum_{t=h}^H r_i(s_t,a_t)$. In RL practice, the policy gradient $\nabla J_i(\theta)$ is replaced by its Monte Carlo (MC) estimation, since evaluating the exact value is intractable. Given a batch of trajectories $\mathcal{D}_i(\theta)$ under the policy $\pi_i(\cdot|\cdot;\theta)$, the MC estimation is 
	$\hat{\nabla} J_i(\theta,\mathcal{D}_i(\theta)):={1}/{|\mathcal{D}_i(\theta)|}\sum_{\tau\in \mathcal{D}_i(\theta)} g_i(\tau;\theta)$.
 
Instead of focusing on individual tasks $\mathcal{T}_i$, MAMRL, as introduced in \citep{finn2017model}, aims to find a good initial policy, also called a meta policy that returns satisfying rewards in expectation when it is updated using limited number of gradient step updates under a new environment $\mathcal{T}'$, sampled from a task distribution $p$. In particular, for only one-step policy gradient update with a step size $\alpha$, the optimization problem of MAMRL \citep{fallah_sgmrl} is 
\begin{align}\label{eq:metarl}
	\max_{\theta\in \R^d} L(\theta)=\mathbb{E}_{\mathcal{T}_i\sim p}\mathbb{E}_{\mathcal{D}_i(\theta)\sim_{q_i(\cdot;\theta)}}[J_i(\theta+\alpha\hat{\nabla}J_i(\theta, \mathcal{D}_i(\theta)))].	
\end{align}
 A stochastic gradient  method (SG-MRL) is proposed in \citep{fallah_sgmrl} for solving the maximization problem \eqref{eq:metarl} (see \Cref{algo:sg-marl}), where an unbiased estimator of $\nabla L(\theta)$ can be efficiently computed for performing gradient ascent.  Suppose that for simplicity $\mathcal{D}^1_i=\{\tau\}$, the data set sampled using $q_i(\cdot;\theta_t)$,  only contains one trajectory, then $\nabla_\theta \E_{\mathcal{D}\sim q_i(\cdot;\theta)}[J_i(\theta+\alpha \hat{\nabla}J(\theta, {\mathcal{D}}))]=\E_{\tau\sim q_i(\cdot;\theta)}[\nabla J_i(\theta+ \alpha g_i(\tau;\theta))(I+\alpha\nabla_\theta g_i(\tau;\theta))+J_i(\theta+ \alpha g_i(\tau;\theta))\nabla_\theta \log\pi(\tau;\theta)]$. The MC estimation of $\nabla_\theta \E_{\mathcal{D}\sim q_i(\cdot;\theta)}[J_i(\theta+\alpha \hat{\nabla}J(\theta, {\mathcal{D}}))]$ is given below:
\begin{align}
	\nabla_\theta^i MC=\hat{\nabla} J_i(\theta+ \alpha g_i(\tau;\theta))(I+\alpha \nabla_\theta g_i(\tau;\theta))+\hat{J}_i(\theta+ \alpha g_i(\tau;\theta))\nabla_\theta \log\pi(\tau;\theta),\label{eq:mc_nabla}
\end{align}
where $\hat{\nabla} J_i(\theta+ \alpha g_i(\tau;\theta))$ and $\hat{J}_i(\theta+ \alpha g_i(\tau;\theta))$ are estimated using another batch of trajectories $\mathcal{D}_i^2$ rolled out under the policy $\pi(\theta+\alpha  g_i(\tau;\theta))$. The average of $\nabla_\theta^i MC$ constitutes an unbiased estimate of $\nabla L(\theta)$. 

\section{Sampling Attack Models}\label{sec:model}
In this section, we propose three sampling attack models and characterize them as Stackelberg games between the attacker and the learner.  As shown in  \Cref{algo:sg-marl}, MAMRL includes two sampling processes. The first sampling (line 5 in  \Cref{algo:sg-marl}) is for the MC estimation of the corresponding policy gradient, while the second sampling (line 7 in \Cref{algo:sg-marl}) is for the MC estimation of the gradient of the objective function $L(\theta)$. The main idea of the sampling attack is to manipulate the reward signal so that MC estimations are corrupted. Since the meta training includes two sampling processes, there are three possible sampling attacks: 1) Inner Sampling Attack (ISA): only attacking the first sampling process, where samples are used to compute the gradient adaptation in the inner loop; 2) Outer Sampling Attack (OSA): only attacking the second sampling process, where samples are used to perform one step meta optimization in the outer loop; 3) Combined Sampling Attack (CSA): attacking both the first and the second sampling process.
 
\begin{minipage}{0.5\textwidth}
	\begin{algorithm}[H]	
	\caption{SG-MARL \citep{fallah_sgmrl}}
	\begin{algorithmic}[1]\label{algo:sg-marl}
		\State \textbf{Input} Initialization $\theta_0$, $t=0$, step size $\alpha,\eta_1$.
		\While{not converge}
		\State Draw a batch of i.i.d tasks $\mathcal{T}_i, i\in \mathcal{I}$ with size $|\mathcal{I}|=I$;
		\For{$i\in\mathcal{I}$ }
		\State Sample a batch of trajectories $\mathcal{D}^1_i$ under $q_i(\cdot;\theta_t)$;
		\State $\theta^i_{t+1}=\theta_t+\alpha \hat{\nabla} J_i(\theta_t,\mathcal{D}^1_i)$ ;
		  
		\State Sample a batch of trajectories $\mathcal{D}_i^2$ under $q_i(\cdot;\theta_{t+1}^i)$;
		\EndFor
		\State Compute $\nabla_\theta^i MC$ using \eqref{eq:mc_nabla}
		\State $\theta_{t+1}=\theta_t+(\eta_1/I) \sum_{i\in\mathcal{I}}\nabla^i_{MC}$
		\EndWhile
		\State \textbf{Return } $\theta_{t+1}$
	\end{algorithmic}
\end{algorithm}
\end{minipage}
\hfill
\begin{minipage}{0.45\textwidth}
	\begin{figure}[H]
	\centering
	\includegraphics[width=1\textwidth]{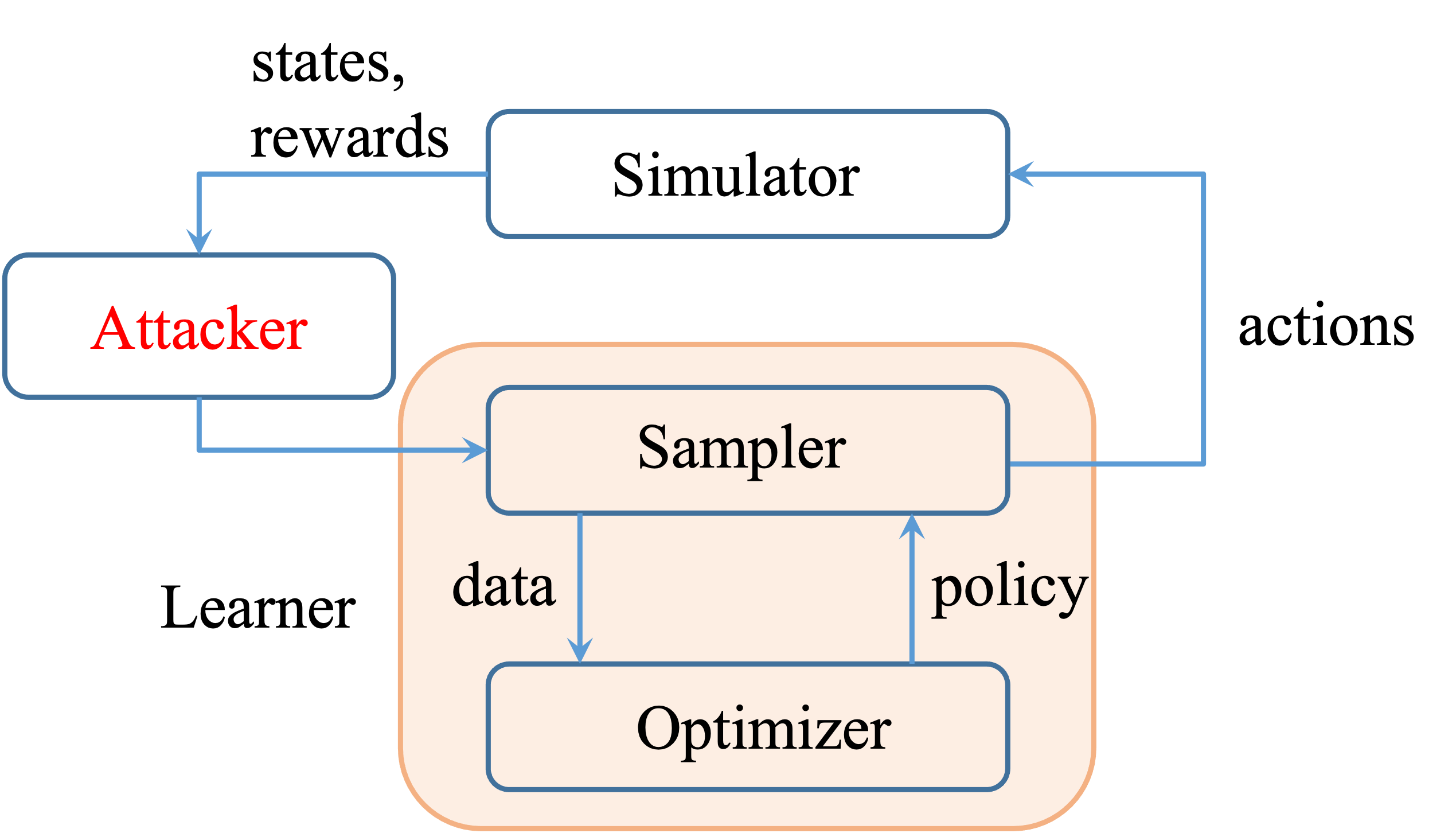}
	\caption{A schematic illustration of sampling attack. The attacker hijacks the reward feedback sent from the simulator to the learner, and poisons the rewards using attack methods (Intermittent and Persistent Attack) introduced in \Cref{sec:method}, in order to mislead the learner. }
	\label{fig:sample_attack}
	\end{figure}
\end{minipage}

\textbf{Practicability of Sampling Attacks} Before delving into the details of these attacks, we first illustrate how this attack can happen in a training environment. For the training process in RL, a learner, as the orange box in \Cref{fig:sample_attack}, usually consists of two components: a sampler and an optimizer. At each state, the sampler chooses an action based on the current policy, and sends it to the simulator, which returns a reward as well as the next state. The state-action-reward tuple will be stored within the sampler, and this process repeats until multiple trajectories are sampled. These sample trajectories make up a training data set fed to the optimizer, which performs one-step policy improvement. 

To launch sampling attacks, the attacker first covertly infiltrates into the learning system shown in \Cref{fig:sample_attack} through Advanced Persistent Threats (APT) processes \citep{zhu18apt}. Details on this infiltration are included in \Cref{app:prac}. After the infiltration, the attacker intercepts communication between the sampler and the simulator and sends the learner poisoned reward signals, misleading the learner. One may argue that the attacker could gain control of the whole system, and its possible maneuvers are more than reward manipulation (e.g., compromising the simulator, corrupting the state input). In \Cref{app:justify}, we justify our choice of reward manipulation as the first step of the investigation without losing the generality of sampling attacks.         
 
The following summarizes three important aspects of the proposed sampling attacks. 1) \textbf{Knowledge of the attacker}: the attacker only has access to the sample trajectories rolled out by the learner and the current policy. The simulator and the learner's internal learning process remain black-box to the attacker. 2)  \textbf{Attack Strategies}: the attacker is allowed to modify the rewards of sample trajectories provided by the simulator. It can adjust its attack sequentially depending on the meta learning process. Note that the attacker needs to choose the manipulation strategy from a constraint set, which can be interpreted as the attack budget (see \Cref{app:budget}).  3) \textbf{The attacker's objective}: this paper considers a greedy attacker who misleads the learner into learning the worst possible meta policy, i.e., the value function of the learned meta policy [$L(\theta)$ in \eqref{eq:metarl}] is minimized after sampling attacks. In the subsequent, we use ISA as an example to elaborate on these aspects of the proposed sampling attacks.   
        
\textbf{Inner Sampling Attack} The first attack is to apply reward manipulations to the first sampling process. Denote the manipulated reward by $r^{adv}(s,a;\delta)$, which is parameterized by $\delta\in \Delta\subset\R^n, n\in \N$. We assume that $r^{adv}(s,a;\delta):=\delta\cdot r(s,a)$, i.e., the manipulated reward is obtained by applying the linear transformation to the original reward, and $\delta\in\Delta\subset \R$. It is a more general practice to represent $r^{adv}$ by a neural network, $\delta$ being the weights of the network. Another remark is that the admissible  manipulation strategies $\delta$ are confined to a bounded set $\Delta$ (see \Cref{ass:domain}), limiting the attacker's ability. \Cref{app:budget} compares our attack constraints with existing constraints in the literature. 

Under ISA, the policy adaptation $\theta_{t+1}^i=\theta_t+\alpha \hat{\nabla}J_i(\theta, \mathcal{D}_i^1)$ is compromised, as $\mathcal{D}_i^1$ is replaced by $\widetilde{\mathcal{D}}_i^{1}$, in which every reward is scaled by $\delta$. For simplicity, it is assumed that reward signals in $\mathcal{D}_i^1$ are uniformly scaled by $\delta$. More sophisticated manipulation strategies, such as adaptive poisoning \citep{zhang2020adaptive}, can also be considered, yet to keep the theoretical analysis neat, we limit ourselves to this uniform scaling within each batch.  In the following, notations with $\widetilde{\mbox{\hspace{4pt}}\cdot\mbox{\hspace{4pt}}}$ either denote data corrupted by the manipulation or quantities computed using corrupted data. 

Following the policy gradient theorem reviewed in \Cref{sec:pre}, the policy gradient under ISA takes the following form $\widetilde{\nabla }J_i(\theta)=\E_{\tau\sim q_i(\cdot;\theta)}[{g}_i(\widetilde{\tau};\theta)]$, where $ 
	\widetilde{g}_i(\tau;\theta)= \sum_{h=1}^H\nabla_\theta\log \pi_i(a_h|s_h;\theta){R}_i^h(\widetilde{\tau})$. Since $\widetilde{\tau}$ only differs from $\tau$ in that the rewards are scaled by $\delta$, ${R}_i^h(\widetilde{\tau})=\delta R_i^h(\tau)$,  the relationship between the corrupted policy gradient and the original one is given by $\widetilde{\nabla J_i}(\theta)=\delta \nabla J_i(\theta)$. Similarly, for the MC estimation, we have $\hat{\nabla}J(\theta, \widetilde{\mathcal{D}})=\delta \hat{\nabla}J(\theta,\mathcal{D})$. Finally, the goal of the attack is to find a $\delta$ such that the performance of the learned meta policy [evaluated through the objective function $L(\theta)$, specified in \eqref{eq:metarl}.] is minimized. The optimization problem associated with (ISA) is given by
\begin{align}\label{eq:inner-attack}
	\min_{\delta\in \Delta}\quad & L(\theta^*) \tag{ISA}\\
	\text{s.t.}\quad &  \theta^*\in \argmax_{\theta\in \R^d}\E_{i\sim p}\E_{\mathcal{D}\sim q_i(\cdot;\theta)}[J_i(\theta+\alpha \cdot\delta\hat{\nabla}J(\theta, {\mathcal{D}}))]. 	\nonumber
\end{align}

\textbf{Outer Sampling Attack}
Another  sampling attack is to consider manipulating the samples $\mathcal{D}_i^2$ in \Cref{algo:sg-marl}. The outer sampling collects samples under the updated policy $\theta_{t+1}^i$, which are used to compute the MC estimation of the gradient of $L(\theta_t)$. In this case, the optimization problem in OSA share the same attack objective $\min_{\delta\in\Delta} L(\theta^*)$ as in ISA, but the constraint becomes 
\begin{align}\label{eq:outer-attack}
	\text{s.t.}\quad & \theta^*\in \argmax_{\theta\in \R^d}\E_{i\sim p}\E_{\mathcal{D}\sim q_i(\cdot;\theta)}[\delta J_i(\theta+\alpha \cdot\hat{\nabla}J(\theta, {\mathcal{D}}))]. \tag{OSA}
\end{align}

\textbf{Combined Sampling Attack}
Finally, the inner and the outer sampling attack can be combined together. Following the same line of argument, the objective is $\min_{\delta\in\Delta} L(\theta^*)$, and the constraint is 
\begin{align}\label{eq:com-attack}
	\text{s.t.} \quad   \theta^*\in \argmax_{\theta\in \R^d}\E_{i\sim p}\E_{\mathcal{D}\sim q_i(\cdot;\theta)}[\delta J_i(\theta+\alpha \cdot\delta\hat{\nabla}J(\theta, {\mathcal{D}}))].\tag{CSA}
\end{align}

Some remarks are in order. Note that compared with the other two sampling attacks,  OSA is rather trivial: a reward flipping, i.e., $\delta=-1$ is sufficient, which makes $\theta^*$ already a minimizer. Hence, for the rest of the paper, we mainly discuss \eqref{eq:inner-attack} and \eqref{eq:com-attack}. Our theoretical analysis focuses on  \eqref{eq:inner-attack}, and its extension to \eqref{eq:com-attack} is straightforward (see \Cref{app:mmp_up}).  Another remark is that  even though there may exist multiple maximizers $\theta^*$, as the objective function is nonconvex differentiable, we assume that in our formulation, $\theta^*$ is properly selected, and hence unique. One way to justify this assumption is that $\theta^*$ can be viewed as the policy parameter returned by the RL algorithm in the training process, and hence, it is unique. We will elaborate on this selection more mathematically (see \Cref{lemma:phi_diff} and its proof) in \Cref{sec:method}. 

Note that all three sampling attacks are formulated as bi-level optimization problems. Hence they amount to a Stackelberg game between the attacker (leader) and the learner (follower). The best response of the learner is to learn an optimal policy under the reward manipulation, and the three constraints on $\theta^*$ in the above formulations  correspond to the learner's meta learning processes under ISA, OSA, and CSA, respectively. Knowing the learner's best response, the attacker aims to find a $\delta$ such that the learned meta policy adapts poorly in the testing phase without further attacks.  

Since the attacker targets the testing performance of $\theta^*$, it needs to consider the policy gradient adaptation $(\theta^*+\alpha \hat{\nabla}J_i(\theta^*, \mathcal{D}(\theta^*))$  to evaluate $L(\theta^*)$, where $\mathcal{D}(\theta^*)$ is the genuine sample trajectories under $\theta^*$ without manipulation. However, in (ISA), the attacker can only access the learner's contaminated training data [i.e., sample trajectories under the policy $(\theta+\alpha \cdot \delta \hat{\nabla}J(\theta, {\mathcal{D}}))$]. In this case, the feedback of the attacker's manipulation [e.g., estimates of $\nabla L(\theta)$] is not directly available during the training time. Hence, searching for the optimal attack defined in \eqref{eq:inner-attack} in an online manner is challenging, due the attacker's limited knowledge of the meta learning process.

\textbf{From Stackelberg to Minimax} As discussed above, it is a daunting task for the attacker to find the optimal $\delta$ by solving the bi-level optimization problem in \eqref{eq:inner-attack}. Thus, we reformulate the attacker's problem as a minimax problem (MMP), which is  more tractable than \eqref{eq:inner-attack}, as the attacker now targets the learner's training performance as shown below. 
 \begin{align}\label{eq:minimax_attack}
	\min_{\delta\in \Delta} \max_{\theta\in \R^d} \E_{i\sim p}\E_{\mathcal{D}\sim q_i(\cdot;\theta)}[ J_i(\theta+\alpha \cdot \delta \hat{\nabla}J(\theta, {\mathcal{D}}))].\tag{MMP}
\end{align}  
{Note that in general the $\min$ and $\max$ operator in \eqref{eq:minimax_attack} is not changable, since the objective is nonconvex-nonconcave \citep{fallah_sgmrl}. Hence, \eqref{eq:minimax_attack} still amounts to a Stackelberg game, yet the attacker now aims to minimize the training performance under ISA, i.e., the value function under the corrupted dataset [the max part of \eqref{eq:minimax_attack}]. Sample trajectories rolled out by the learner can also help guide the attacker's search for the optimal attack, and no extra data or information is needed. In this case, the attack freerides the meta learing process, making the proposed sampling attack stealthy and lightweight. This free-ride nature will be more clear in \Cref{sec:method}, where two online attack schemes based on the gradient feedback are introduced.   
}

Before concluding this section, we comment on the relationship between original \eqref{eq:inner-attack} and \eqref{eq:minimax_attack}. Even though the minimax value of \eqref{eq:minimax_attack} is by definition different from the Stackelberg equilibrium payoff given by \eqref{eq:inner-attack}, our following \Cref{lem:upper-bound} states that under the assumption of bounded policy gradients (see \Cref{ass:bound_pg} in \Cref{app:regular}), the minimax value, when combined with a $\ell_2$-norm regularization $\|\delta-1\|$ serves as an upper-bound of the value function under the optimal attack.  

\begin{proposition}\label{lem:upper-bound}
Under the assumption of bounded policy gradients, there exists a constant $C$, such that for $\theta^*=\theta^*(\delta)$ defined in \eqref{eq:inner-attack}, any batch of trajectories $\mathcal{D}$,
	\begin{align}
	 J_i(\theta^*+\alpha \hat{\nabla}J_i(\theta^*, \mathcal{D}))]\leq &  J_i(\theta^*+\alpha \cdot \delta \hat{\nabla}J(\theta^*, {\mathcal{D}}))]+C\|\delta-1\|.\label{eq:pre_upper}
	\end{align}
	As a consequence, 
	\begin{align}	
	\min_{\delta\in \Delta}L(\theta^*)\leq \min_{\delta\in \Delta} \E_{i\sim p}\E_{\mathcal{D}\sim q_i(\cdot;\theta)}[ J_i(\theta^*+\alpha \cdot \delta \hat{\nabla}J(\theta^*, {\mathcal{D}}))]+ C\|\delta-1\|.\label{eq:stacktominimax}
	\end{align}
\end{proposition}
From \Cref{lem:upper-bound}, the free-rider nature can explained in the following way: as long as the training performance, i.e., the right-hand side of \eqref{eq:stacktominimax} deteriorates, the testing performance or the quality of the meta policy, i.e., the left-hand side of \eqref{eq:stacktominimax} cannot be any better. Even though the proposed sampling attack models are discussed in the context of MAMRL, they are also relevant to other meta RL frameworks, as discussed in \Cref{app:relevance}.

\section{Sampling Attacks under Concurrent Learning }\label{sec:method}
Based on \eqref{eq:minimax_attack}, this section presents two onine attack schemes that enable the attacker to learn the optimal attack alongside with learner's learning process. We present the two schemes in the context of ISA, and the extension to CSA is straightforward (see \Cref{app:mmp_up}). To streamline the investigation into the optimality of proposed mechanisms, we skip the $\ell_2$ term, treating it as a regularization in numerical implementations, and mainly focus on  \eqref{eq:minimax_attack}.

\textbf{Intermittent Attack} In  \eqref{eq:minimax_attack}, the inner maximization problem $\max_{\theta\in \R^d}\E_{i\sim p}\E_{\mathcal{D}\sim q_i(\cdot;\theta)}[J_i(\theta+\alpha \cdot\delta\hat{\nabla}J(\theta, {\mathcal{D}}))]$ corresponds to the meta RL under ISA, while the attacker's problem is given by $\min_{\delta\in \Delta} \E_{i\sim p}\E_{\mathcal{D}\sim q_i(\cdot;\theta)}[J_i(\theta^*+\alpha \cdot\delta\hat{\nabla}J(\theta^*, {\mathcal{D}}))]$. A natural way to seek the minimizer is to apply gradient descent to the objective function.
Using the similar argument in \eqref{eq:mc_nabla}, we arrive at the following MC estimation of the gradient with respect to $\theta$, 
\begin{align}\label{eq:attack_mc_theta}
\widetilde{\nabla}_\theta^i MC=\hat{\nabla} J_i(\theta+\delta g_i(\tau;\theta))(I+\delta\nabla_\theta g_i(\tau;\theta))+\hat{J}_i(\theta+\delta g_i(\tau;\theta))\nabla_\theta \log\pi(\tau;\theta).
\end{align}
 Similarly for the gradient with respect to $\delta$ when fixing $\theta^*$, the MC estimate is given by  
\begin{align}
\nabla_\delta^i MC= \alpha \cdot \hat{\nabla} J_i(\theta^*+\alpha\cdot \delta g_i(\tau;\theta^*))g_i(\tau;\theta^*)\label{eq:mc_delta}
\end{align}

The attacker's learning proceeds as follows. At first, the attacker initializes the attack by setting $\delta=1$, and closely monitors the learning process. During the meta training process, the attacker remains dormant, i.e., $\delta$ remains constant, until policy parameter $\theta$ stabilizes. Then it is activated and updates $\delta$ using one-step stochastic gradient descent under the Euclidean projection $\Pi_{\Delta}(\cdot)$.  It repeats the above steps until the $\delta$ stabilizes. Since the attacker only adapts the reward manipulation intermittently, we refer to such attack as Intermittent Attack. The pseudocode is in \Cref{algo:intermit}. 

\textbf{Persistent Attack}
Instead of waiting for the learner to converge to $\theta^*(\delta)$, the attacker can also implement a real-time attack, i.e., it performs gradient descent along with the learner's gradient ascent, no matter whether the current $\theta$ stabilizes or not. In this case, with the initialization $\delta=1$, the attacker is spared from detecting the stabilization of the learner's training and updates $\delta$ every time $\theta$ is updated. \Cref{algo:persist} summarizes this kind of attack, which we refer to as Persistent Attack, since the attacker keeps adjusting $\delta$ according to the learner's progress all the time. The intuition behind Persistent Attack is that if the attacker's learning rate is far smaller than the learner's, the learner views the attack as quasi-static, while the attacker treats $\theta_t$ as stabilized \citep{lin_twotime}, which shares the same spirit as Intermittent Attack.       

Before proceeding to the theoretical analysis, we comment on the differences between the two schemes in terms of  cost and stealthiness. Since Intermittent Attack utilizes the stabilized policy parameter, the attacker's MC estimation \eqref{eq:mc_delta} in Intermittent Attack tends to return a more accurate gradient descent direction than its counterpart in Persistent Attack when minimizing the objective function in \eqref{eq:minimax_attack}. For the attacker's gradient estimation under two attack schemes, the reader is referred to \Cref{lemma:phi_diff}, \Cref{lemma:decrease_d}, and associated proofs. As a result, in practice, Intermittent Attack tends to find the optimal attack using fewer iterations (as shown in \Cref{fig:experi}(b)(d).) However, in Intermittent Attack, the attacker needs to detect the stabilization of the learner's training. This detection can be costly, as the attacker needs to keep a record of the average return of each iteration. In contrast, in Persistent Attack, the attacker simply updates the attack $\delta$ per iteration without detection. On the other hand, since Persistent Attack manipulates the rewards of sample trajectories during each iteration, it is less stealthy than Intermittent Attack. 

\vspace{-0.5cm}
\begin{minipage}[t]{0.49\textwidth}
	\begin{algorithm}[H]
	\caption{Intermittent Attack}\label{algo:intermit}
	\begin{algorithmic}
		\State\textbf{Input} maximum iterations $N, K$, meta learning step size $\eta_1$, meta attack step size $\eta_2$, initialization $\delta_0, \theta_0$, policy gradient step size $\alpha$.
		\For{$t=0,\ldots,N-1$}
		\State $\theta_0(\delta_t)=\theta_t$
		\For{$k=0,\ldots, K-1$}
		\State Draw a batch of tasks $\mathcal{T}_i$, $i\in \mathcal{I}$;
		\For{$i\in \mathcal{I}$ }
		\State Sample  $\mathcal{D}^1_i$ from $q_i(\cdot;\theta_t)$;
		\State $\theta^i_{k+1}=\theta_k+\alpha\cdot \delta \hat{\nabla} J_i(\theta_k,\mathcal{D}^1_i)$;
		\State Sample $\mathcal{D}^2_i$ from $q_i(\cdot;\theta^i_{k+1})$;
		%\State Compute $\widetilde{\nabla}^i_\theta{MC}$ using \eqref{eq:attack_mc_theta};
		\EndFor  
		\State $\theta_{k+1}=\theta_k+(\eta_1/I) \sum_{i\in\mathcal{I}}\widetilde{\nabla}^i_\theta{MC}$; 
		\EndFor
		\State $\theta_{t+1}= \theta_K(\delta_t)$;
		%\State Compute the MC estimation $\nabla^i_\delta{MC}$ using \eqref{eq:mc_delta};
		\State $\delta_{t+1}=\Pi_{\Delta}\left(\delta_t+(\eta_2/I) \sum_{i\in \mathcal{I}}\nabla^i_\delta{MC}\right)$ ;  
		\EndFor
		\State \textbf{Return} $\{\delta_{t},\theta_K(\delta_t)\}$ for $t=0,1,\ldots, N$.
	\end{algorithmic}
\end{algorithm}
\end{minipage}
\hfill
\begin{minipage}[t]{0.49\textwidth}
	\begin{algorithm}[H]
	\caption{Persistent Attack}
	\begin{algorithmic}
		\State \textbf{Input} Initialization $(\delta_0,\theta_0)$, \\
		step sizes $\eta_1>\eta_2$.
		\For {$t=0, 1, 2,\ldots, N-1$}
		\State Draw a batch of iid tasks $\mathcal{T}_i$, $i\in \mathcal{I}$ with size $I$;
		\For{$i\in I$ }
		\State Sample  $\mathcal{D}^1_i$ with respect to $q_i(\cdot;\theta_t)$;
		\State $\theta^i_{t+1}=\theta_t+\alpha \cdot \delta \hat{\nabla} J_i(\theta_t,\mathcal{D}^1_i)$;
		\State Sample  $\mathcal{D}^2_i$ with respect to $q_i(\cdot;\theta^i_{t+1})$;
		\State Compute  $\widetilde{\nabla}^i_\theta{MC}$ using \eqref{eq:attack_mc_theta}, and $\nabla^i_\delta{MC}$ using \eqref{eq:mc_delta};
		\EndFor  
		\State $\theta_{t+1}=\theta_t+(\eta_1/I) \sum_{i\in\mathcal{I}}\widetilde{\nabla}^i_\theta{MC}$;
		\State $\delta_{t+1}=\Pi_{\Delta}\left(\delta_t-(\eta_2/I) \sum_{i\in\mathcal{I}}\nabla^i_\delta{MC}\right)$;
		\Comment{ $\Pi_{\Delta}(\cdot)$ is the Euclidean projection}
		\EndFor
		\State \textbf{Return} \{$\delta_t,\theta_t\}$ for $t=0,1,\dots, N$.
	\end{algorithmic}\label{algo:persist}
\end{algorithm}
\end{minipage}

\textbf{Optimality of Sampling Attacks}
The objective in \eqref{eq:minimax_attack} is nonconvex-nonconcave. An appropriate optimality condition is the $\epsilon$-first order stationary point ($\epsilon$-FOSP), widely adopted in the study of programming and game theory \citep{gda_two,li2021confluence}.
\begin{definition}[$\epsilon$-FOSP]\label{def:fosp}
 $(\delta^*,\theta^*)$ is an $\epsilon$-FOSP of a continuously differentiable function $f(\delta,\theta) $ if $\min_{\delta\in B(\delta^*,1)} \left\langle\nabla_\delta f(\delta^*,\theta^*),\delta-\delta^*\right\rangle \geq -\epsilon$, $\max_{\theta\in B(\theta^*,1)}\left\langle\nabla_{\theta}f(\delta^*,\theta^*),\theta-\theta^*\right\rangle\leq \epsilon$, 
	 where $B(\delta^*,1):=\{\delta\in \Delta: \|\delta-\delta^*\|\leq 1\}$, $B(\theta^*,1):=\{\theta\in\R^d: \|\theta-\theta^*\|\leq 1\}$. 
\end{definition} 
Since the objective function in \eqref{eq:minimax_attack} is noncovex-nonconcave, in order to show the optimality of the attack, we need to impose some regularity conditions. For the ease of notation, we define $f(\delta, \theta):=\E_{i\sim p}\E_{\mathcal{D}\sim q_i(\cdot;\theta)}[J_i(\theta+\alpha \cdot\delta\hat{\nabla}J(\theta, {\mathcal{D}}))]$, and denote its MC estimation by $\hat{\nabla}f(\delta,\theta;\xi)$, where the random variable $\xi$ represents the a collection of datasets $\{\mathcal{D}_1^i\}_{i\in \mathcal{I}}, \{\mathcal{D}_i^2\}_{i\in \mathcal{I}}$ to be sampled in the sampling processes for the computation of $\nabla_\delta^i MC , \nabla_\theta^i MC$. Denote $\Phi(\delta):=\max_{\theta}f(\delta,\theta)$. 
The following assumptions extend Polyak-\L ojasiewicz (PL) condition \citep{polyak1963gradient,loja} and restricted secant inequality (RSI) \citep{zhang2013gradient} to the MMP. Note that both \Cref{ass:minmaxpl} and \Cref{ass:rsi} are weaker than commonly assumed strong convexity condition \citep{lin_twotime,gda_two}, and detailed discussions are presented in \Cref{app:regular}. 
\begin{assumption}[Minmax PL]\label{ass:minmaxpl}
	For  \eqref{eq:minimax_attack}, it is assumed that there exists a constant $\mu>0$ such that  $\frac{1}{2}\|\nabla_\theta f(\delta,\theta)\|^2\geq \mu(\max_{\theta}f(\delta,\theta)-f(\delta,\theta))$.
\end{assumption}
\begin{assumption}[Minimax RSI]\label{ass:rsi}
 	For \eqref{eq:minimax_attack}, it is assumed that there exists a constant $\mu>0$ such that for any fixed $\delta$,
 	$\left\langle \nabla_\theta f(\delta,\theta), \theta^*(\delta)-\theta \right\rangle \geq \mu\|\theta^*(\delta)-\theta\|^2, \forall \theta\in \R^d$.	
 \end{assumption}
 In addition to the above assumptions, some regularity assumptions are also needed to show the convergence of  Intermittent and Persistent Attack. They are presented in \Cref{app:regular}. 
%In addition, we assume that the stochastic gradient $\hat{\nabla}f(\delta,\theta;\xi)$ is unbiased, and of finite variance (see \Cref{ass:sto}), and that $f(\delta,\theta)$ is of Lipschitz smoothness (see \Cref{ass:lip}), which is customary in the RL literature \citep{fallah_sgmrl}. Finally, it is assumed that the domain $\Delta$ is a bounded, convex, and compact set (see \Cref{ass:domain}). Details of these regularity assumptions can be found in the appendix.

\textbf{Nonasymptotic Analysis of Intermittent Attack}\label{sec:analysis_intermittent}
The convergence result rests on that when the meta learning stabilizes, the attacker's gradient feedback $\nabla_\delta f(\delta_t, \theta^*_t(\delta_t))$ equals $\nabla\Phi(\delta_t)$ (see \Cref{lemma:phi_diff}). Then, the standard analysis of gradient descent in nonconvex programming can be applied.
 
\begin{theorem}\label{thm:intermittent}
	Under \Cref{ass:minmaxpl} and regularity assumptions, for any given $\epsilon\in (0,1)$, let the step sizes $\eta_1,\eta_2$ as well as the batch size $M$ be properly chosen (see \Cref{app:inter_proof}), if $N\geq N(\epsilon)\sim \mathcal{O}(\epsilon^{-2})$, $K\geq K(\epsilon)\sim \mathcal{O}(\log\epsilon^{-1})$, then there exists an index $t$ such that $\{\delta_t,\theta_K(\delta_t)\}$ generated by \Cref{algo:intermit} is an $\epsilon$-FOSP.
\end{theorem}

\textbf{Nonasymptotic Analysis of Persistent Attack}\label{sec:analysis_persistent}
Although \Cref{ass:minmaxpl} suffices for proving the optimality of Intermittent Attack, it is rather a weak condition when dealing with Persistent Attack. The gap is that in Persistent Attack, the attacker's gradient update does not utilize the exact gradient $\nabla\Phi(\delta_t)=\nabla_\delta f(\delta_t,\theta^*_t)$ or its approximation $\nabla_\delta f(\delta_t,\theta_K(\delta_t))$, instead, $\nabla_\delta f(\delta_t,\theta_t)$ is applied. To control the difference $d_t=\|\theta^*_t-\theta_t\|$, we show in \Cref{app:per_proof} that $d_t$ exhibits a linear contraction using \Cref{ass:rsi}, leading to the convergence result stated as follows. 

\begin{theorem}\label{thm:persistent}
Under \Cref{ass:rsi} and  regularity conditions, for any given $\epsilon\in (0,1)$, let the step sizes $\eta_1,\eta_2$ as well as the batch size $M$ be properly chosen (see \Cref{app:per_proof}),  then if $N\geq N(\epsilon)\sim \mathcal{O}(\epsilon^{-2})$,  there exists an index $t$, such that $\{\delta_t,\theta_t\}$ generated by \Cref{algo:persist} is an $\epsilon$-FOSP. 
\end{theorem}

\section{Numerical Experiments}\label{sec:num}
This section presents experimental evaluations of the proposed attack schemes for ISA and CSA. We consider two benchmark meta learning tasks, including a 2D-navigation task and a more complex locomotion problem: Half-cheetah (H-C) \citep{finn2017model}. The following briefly introduces the two tasks, and more details can be found in \Cref{app:meta_task}. In the 2D-navigation task, starting from a fixed initial position [the black triangle in \Cref{fig:experi}(a)] in a 2D plane, a point agent [the red dot in \Cref{fig:experi}(a)] needs to move to a goal position [the green star in \Cref{fig:experi}(a)] randomly sampled from a given task distribution. In the H-C task, a planar cheetah [as shown in \Cref{fig:experi}(c)] is required to run at a particular velocity sampled from the task distribution.

%For Intermittent Attack in \Cref{algo:intermit} and Persistent Attack in \Cref{algo:persist}, the meta learning rate $\eta_1=0.1$ and the attack learning rate $\eta_2=0.01$ for both 2D and H-C tasks. 
\textbf{Experimental Setup} Hyperparameter setup for the two algorithms is detailed in \Cref{app:setup}. When evaluating the corrupted meta policy in the testing phase, we compare the adaptation performance [$L(\theta)$ defined in \eqref{eq:metarl}] of meta policies learned under different attack settings and the benchmark setting without attacks.  One-step policy adaption is considered throughout the paper. Additional experiments on multi-step adaption can be found in \Cref{app:add_exp}. As in \citep{finn2017model}, the evaluation metric is the average of cumulative rewards of policies adapted from the meta policy.

\textbf{Experiment Results} \Cref{fig:experi} presents an example of the poor adaptation of the corrupted meta policies and the convergence of two online attack schemes under ISA in a single experiment. In \Cref{fig:experi}(a), when using the policy adapted from the meta policy, the agent is misled into searching the lower-right corner [as the search paths concentrate on that corner], opposite to the true goal. Similarly, with a small perturbation $\delta=0.967$, the cheetah agent fails to learn the desired running-forward policy as shown in \Cref{fig:experi}(c). In this example, it is observed in \Cref{fig:experi}(b)(d) that both attack schemes converge. Intermittent Attack takes fewer iterations since its gradient estimation is more accurate as we discussed in \Cref{sec:analysis_intermittent}. To demonstrate the impact of sampling attacks, the first two columns in \Cref{tab:num_res} summarizes the average rewards (and their standard deviations) of the policies after one-step adaptation from the meta policies corrupted by different attacks.
\begin{table}
\centering
\caption{Experimental evaluations of Intermittent Attack and Persistent Attack under ISA and CSA as well as robust training under ISA. Average cumulative rewards (and standard deviations) are  chosen as the evaluation metric (higher the better), which is the sample mean of $L(\theta)$, $\theta$ being the meta policies corrupted by different attacks. As shown in the first two columns, Intermittent and Persistent Attacks significantly deteriorate the adaption performance of meta policies. Unlike Persistent Attack ($\eta_1=0.1> 0.01=\eta_2$), in the robust training, the learner's step size $\eta_1=0.01$ is smaller that the attacker's $\eta_2=0.1$. The decreased step size makes the meta leanring process robust to attacks.}
	\begin{tabular}{cccccc}
	\toprule 
	& & Intermittent Attack & Persistent Attack & Robust Training & Benchmark\\
	\midrule
	\multirow{2}{*}{2D}& ISA& $-30.06\pm 7.37$& $-298.77\pm 34.09$& $-11.50\pm 0.57$ &\multirow{2}{*}{$-10.92\pm0.85$}\\ 
	\cmidrule{2-5}
	& CSA &$-155.84\pm 3.96$ & $-488.97\pm 31.06$&  $---$&\\
	\midrule
	\multirow{2}{*}{H-C}& ISA & $-129.25\pm 12.69$&$-101.13\pm 10.57$& $-63.26\pm 8.69$ &\multirow{2}{*}{$-63.41\pm 6.86$}\\
	\cmidrule{2-5}
	& CSA & $-98.92\pm 10.85$& $-93.65\pm 9.81$& $---$ &\\
	\bottomrule
	\end{tabular}\label{tab:num_res}  
\end{table}   

\textbf{Robust Training against Sampling Attacks} In Persistent Attack, the attacker and the learner operate in a two-timescale manner. Intuitively, the learner first achieves the inner maximization using a larger step size. On the contrary, if the learner's step size is smaller than the attacker's, then the attacker would first achieve the minimization. In this case, the Stackelberg game becomes a maxmin problem $\max_{\theta\in \R^d}\min_{\delta\in \Delta}f(\delta,\theta)$, inspiring a robust training method. During the meta learning process, the learner can decrease $\eta_1$ when it is suspicious of possible attacks. The obtained meta policy returns the best possible adaptation performance when the training data is corrupted. In one of our experiments, when $0.01=\eta_1< \eta_2=0.1$, it is observed that the meta policy is robust to Persistent Attack under ISA, as the adaptation performance is close to the benchmark. The column ``Robust Training'' in \Cref{tab:num_res} compares the performance of the meta policies trained under Persistent Attack [$(\eta_1, \eta_2)=(0.1,0.01)$], Robust Training [$(\eta_1, \eta_2)=(0.01,0.1)$], and the benchmark.  
       
\begin{figure*}
\centering
	\subfigure[The failure of 2D Navigation task under $\delta=-12.96$. ]{\includegraphics[width=0.23\textwidth]{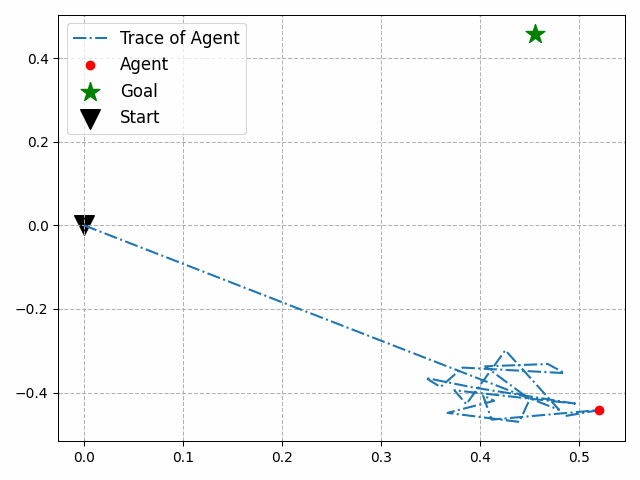}}
	\hfill
	\subfigure[Convergence of $\delta$ under two attacks schemes in 2D navigation.]{\includegraphics[width=0.25\textwidth]{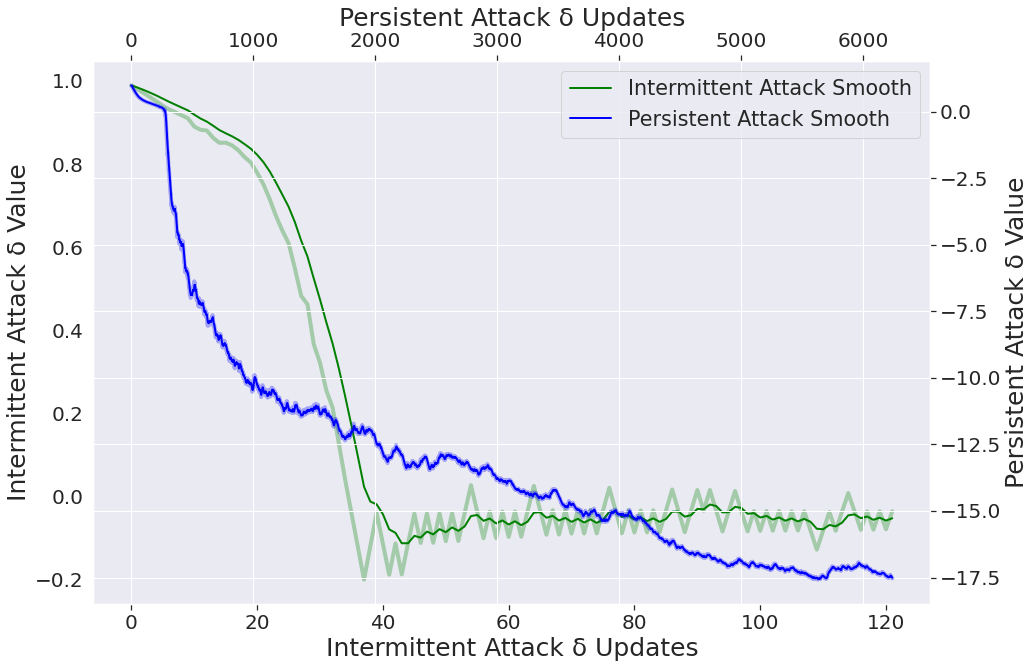}}
	\hfill
	\subfigure[The failure of Half-cheetah task when $\delta=0.967$.]{\includegraphics[width=0.23\textwidth]{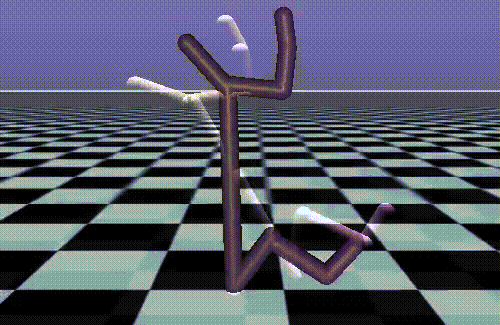}}
	\hfill
	\subfigure[Convergence of $\delta$ under two attacks schemes in Half-cheetah.]{\includegraphics[width=0.25\textwidth]{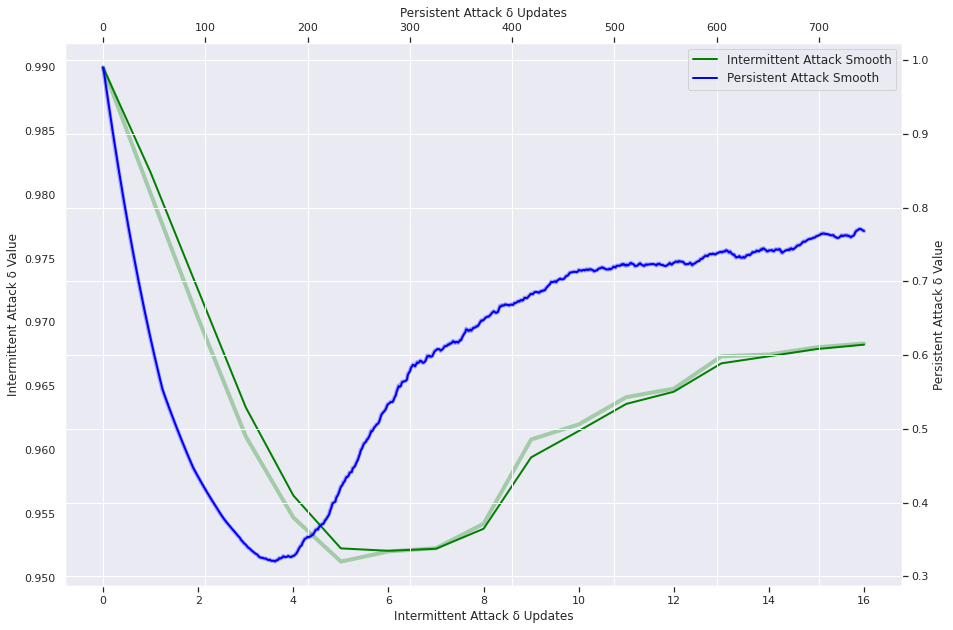}}
	\caption{Experimental results of meta RL agent in 2D navigation task and Half-cheetah task under Intermittent and Persistent Attack under ISA. (a)(c) demonstrates the poor performance of policies adapted from the corrupted meta policy. In (a), the agent is misled to search the lower-right corner, opposite to the true goal. In (c), the agent fails to run forward using the policy adapted from the meta policy. (b)(d) shows that Intermittent Attack need fewer iterations to converge than Persistent Attack.} 
	\label{fig:experi}
\end{figure*}

\section{Discussion}\label{sec:dis}
This work formally defines three sampling attack models on meta RL in the training phase. The optimal sampling attack is formulated as the solution to a minimax problem, which leads to two online black-box attack mechanisms: Intermittent Attack and Persistent Attack. Experiments show that simply scaling the reward feedback during the training, the attacker can significantly deteriorate the adaptation performance of the learned meta policy. 

The main limitation of the proposed online attack schemes is that with the noised gradient feedback, the attacker searches for the $\epsilon$-FOSP (see \Cref{def:fosp}), which is a local notion.  Since the objective function $f(\delta,\theta)$ is nonconvex-nonconcave, it is challenging to characterize these stationary points and compare the corresponding objective values with the global optimum. Hence, quantifying the impact of sampling attacks theoretically (e.g., the performance loss or robustness) remains an open problem. Finally, we comment on the potential negative societal impacts of the proposed sampling attacks. As meta RL has been applied to many real-world applications \citep{meta_survey}, the  stealthy and lightweight sampling attack schemes pose a great threat to these learning systems. Our robust training idea points out a promising way to combat the attack, yet more efforts are needed for a systematic investigation of defense mechanisms.

\bibliographystyle{abbrvnat}
\bibliography{ref}

\newpage

\appendix

\section{Practicability of Sampling Attacks}
\subsection{Infiltration through Advanced Persistent Threats}\label{app:prac}
 In reality, sampling attacks can be achieved through Advanced Persistent Threats \citep{zhu18apt}, a class of multi-stage stealthy attack processes.  In the first stage, the attacker embeds malware in neural network models by replacing the model parameters with malware bytes \cite{liu20stegonet}.  Then, the malware is delivered covertly along with the model into the system in Figure 1. In the second stage, the malware performs reconnaissance and tailors its attack to the communication channel between the simulator and the sampler. In the final stage, the malware launches sampling attacks using either Intermittent Attack (see \Cref{algo:intermit}) or Persistent Attack (see \Cref{algo:persist}). When implementing sampling attacks, the attacker stands between the simulator and the learner, cutting off the information transmission between the two and sending the learner manipulated reward feedback to mislead the learner.

\subsection{Justification of the Reward Manipulation} \label{app:justify} 
%TODO only consider poisoning samples; only attack rewards.
When it has covertly infiltrated the learning system, the attacker can corrupt any component within the learning system in \Cref{fig:sample_attack}. This work considers sampling attacks, a particular case of man-in-the-middle attacks, where the attacker interrupts an existing conversation or data transfer between the simulator and the sampler. Sampling attacks only require sample trajectories under the current policy model to sabotage the learning process. In contrast, if the attacker were to control the simulator or other components of the system, it requires knowledge of how these components are built and operate. Therefore, sampling attacks are more lightweight and stealthy, requiring less knowledge of the learning system. 

The attacker can manipulate everything in sample trajectories during sampling attacks, including states, actions, and rewards. We choose reward manipulation as the first step of the investigation because rewards are one-dimensional scalars, and they are more vulnerable to adversarial manipulations (state and action manipulations may require more advanced devices \citep{zhang2020robust,russoACC21state_perturb} than simply linear scaling). Its analysis leads to sufficient insights without losing the generality of sampling attacks, since the Stackelberg formulation also applies to state and action manipulations.
\subsection{Attack Budgets}\label{app:budget}
%TODO two kinds of budgets, and online attack mainly consider magnititude budgets.
The attack budget aims to limit the attacker's ability to sabotage the learning process; otherwise, it is not surprising that an omnipotent attacker can compromise the learning system of interest. Previous works mainly study two kinds of attack budgets. The first one is to restrict the amount of sample data the attacker can attack \citep{zhang21npg,yunhan19,ma2019policy} which we refer to as the volume constraint. This volume constraint is often considered in offline RL or batch RL scenarios where the training dataset is prefixed \citep{zhang21npg,ma2019policy}. The other one is to confine the attacker's perturbation to a limited magnitude \citep{zhang2020robust,russoACC21state_perturb,zhang2020adaptive}, which we refer to as the magnitude constraint. 

This work considers the magnitude constraint as in other studies on online attacks \citep{zhang2020robust,zhang2020adaptive}. Note that the proposed Intermittent and Persistent Attacks can accommodate the volume constraint by restricting the number of sample trajectories the attacker can poison. In this case, our Stackelberg formulations are still valid, yet the optimality analysis needs refinements as the gradient-based analysis is not directly available. The associated optimality analysis is beyond the nonconvex-nonconcave programming framework and remains an open problem.     

\section{Relevance to Broader Meta RL research}\label{app:relevance}
%TODO max part can be other meta RL; meta training may be value-based?
This work chooses a well-accepted meta RL framework \citep{finn2017model,fallah_sgmrl} and studies associated sampling attacks. In particular, our convergence analysis in \Cref{sec:analysis_intermittent} is built on the theoretical results regarding SG-MRL established in \citep{fallah_sgmrl}. 

Picking MAMRL serves as a starting point, and our findings can apply to many other variants of meta RL algorithms.  The proposed minimax formulation in this work remains valid for other non-MAMRL-based algorithms. Recall that in (MMP), the max part corresponds to the objective in MAMRL. We can replace this objective with its counterparts using other meta RL formulations, such as recurrent-network-based ones \citep{16learnRL,abbeel16rl2}. Similar to MAMRL, meta RL methods in \citep{16learnRL,abbeel16rl2} utilize on-policy data during both meta-training and adaptation. In this case, the attacker can also leverage these on-policy data to estimate its gradient feedback. Hence, Intermittent and Persistent Attacks also apply to these meta RL frameworks.  

On the other hand, our online attack schemes need refinements when dealing with meta RL methods based on off-policy learning \citep{Fakoor2020,rakelly19a}. In off-policy learning, such as Q-learning \citep{Watkins:1992jx} and its variants \citep{silver16go,tao_multiRL,tao_blackwell}, sample trajectories are collected using a behavior policy different from the target policy (the optimal policy). In this case, the attacker may not be able to assess its attack impact on the target policy if it can only access sample trajectories under the behavior policy. One possible remedy would be to reveal to the attacker some side information on the importance ratio between the target policy and the behavior policy \citep{bannon2020causality} so that the target policy's performance can be estimated using off-policy data. We leave this topic as future work. 

\section{Experiment Details}\label{app:experiment}
\subsection{Meta Reinforcement Learning Tasks}\label{app:meta_task}
\paragraph{2D Navigation} Consider in a 2D plane, a point agent must move to different goal positions that are sampled according to a task distribution. Goals are picked from a unit square, centered at the origin with length and width ranging from [-0.5, 0.5]. The state observation is the current 2D position, a 2-dimensional vector, and actions correspond to velocity commands clipped to the range [-0.1, 0.1]. The reward is the negative squared distance to the goal, and iterations terminate when the agent is within 0.01 of the goal or at the horizon of 100 steps. The discounting factor is $0.99$.
\paragraph{Half-Cheetah} In the half-cheetah goal velocity problem, a planar cheetah is required to run at a particular velocity, randomly sampled from $[0,2]$. The state variable is a 17-dimensional vector, representing the position and velocity of the cheetah's joints, and actions correspond to continuous momentum of six leg joints. The transition follows the physical rule simulated by MuJoCo \cite{mujoco}, and the reward is defined as negative absolute value between the current velocity of the agent and a goal, which is chosen uniformly at random from [0, 2]. The horizon length is 200. The discounting factor is $0.99$. 

%\paragraph{Other Mujoco Tasks}
%The two meta RL tasks mentioned above serve as the testbed in this work. However, the proposed attack schemes also apply to other Mujoco tasks, such as the locomotion task of a 3D quadruped in \citep{finn2017model}. The broad applicability of the attacks is due to their black-box nature: the attacker does not require any knowledge of the training environment. 
 
\subsection{Experimental Setup} \label{app:setup}
%TODO  tristen GPU; meta testing; penalty;  
The attack program is built on the meta RL program developed in \citep{mehta20curriculum}. We modify the original meta RL program to support parallel computing using graphics cards, saving a significant amount of training time. All the experiments are conducted using a Linux (Ubuntu 18.04) desktop with 4 Nvidia RTX8000 graphics cards and an AMD Threadripper 3990X (64 Cores, 2.90 GHz) processor. The experiments use a feedforward neural network policy with two 64-unit hidden layers and tanh activiations. The rest of this subsection summarizes the experimental setup of the meta training and sampling attack schemes. 

\paragraph{Meta RL Training} 
Recall that in \Cref{algo:sg-marl}, a batch of i.i.d. tasks are first sampled from the task distribution. This batch of tasks is referred to as the meta batch \citep{finn2017model,fallah_sgmrl}. Then for each task $\mathcal{T}_i$, the sampler returns two batches of trajectories: the inner batch $\mathcal{D}_i^1$ and the outer batch $\mathcal{D}_i^2$. The batch sizes (the number of trajectories/ episodes) of the inner and the outer batch are the same in the experiments. This size is referred to as the batch size. 

In the training phase, the total iteration number is 1000 for 2D navigation (2D) and half-cheetah (H-C) tasks, and the meta batch size is 20 for 2D and 40 for H-C. For both tasks, the batch size is 20 episodes. The learning rate is $\eta_1=0.1$. The policy gradient is obtained through the generalized advantage estimator (GAE) \citep{gae}, which is a refinement of the vanilla policy gradient reviewed in \Cref{sec:pre}. Unlike the original policy gradient estimator $\nabla J(\theta)=\E_{\tau \sim q(\cdot;\theta)}[\sum_{h=1}^H \nabla_\theta \log \pi (a_h|s_h;\theta)R^h(\tau)]$, the GAE estimator takes the following form
\begin{align*}
	\nabla^{GAE} J(\theta)&=\E_{\tau \sim q(\cdot;\theta)}[\sum_{h=1}^H \nabla_\theta \log \pi (a_h|s_h;\theta)A(s_h)],\\
	&=\nabla_\theta \E_{\tau \sim q(\cdot;\theta)}[\sum_{h=1}^H \log \pi (a_h|s_h;\theta)A(s_h)],
\end{align*}    
where the advantage function $A(\cdot)$ is defined as $A(s_h):=R^h(\tau)-\hat{V}(s_h)$, and $\hat{V}(\cdot)$ is the baseline function. $\E_{\tau \sim q(\cdot;\theta)}[\sum_{h=1}^H \log \pi (a_h|s_h;\theta)A(s_h)]$ (its MC estimation) is referred to as the reinforce loss function (empirical reinforce loss), and this loss is the optimization objective maximizing the probability of taking good actions. In short, the higher the reinforce loss, the better policy is. The introduction of the advantage function helps reduce the estimation variance in the RL training. The reader is referred to \citep{gae} for more details on computing the baseline function and interpretations of the advantage function.

\paragraph{Sampling Attacks}
For both Intermittent (\Cref{algo:intermit}) and Persistent (\Cref{algo:persist}) Attacks, the learner's meta learning process uses the hyperparameters specified above, albeit the inner batch and the outer batch are poisoned according to sampling attack schemes. Unless otherwise specified, the learner's and the attacker's learning rates in Intermittent Attack are $\eta_1=0.1, \eta_2=0.1$, respectively, and $\eta_1=0.1, \eta_2=0.01$ in Persistent Attack. In sampling attack experiments, the  regularization $|\delta-1|$ is incorporated into the \eqref{eq:minimax_attack} objective function.
\paragraph{Testing}
To evaluate the quality of the learned meta policy, we conduct test experiments. In the experiment, a batch of tasks is first sampled. The meta policy first perform the gradient adaptation using the inner batch in each sampled task. Then, the adaptation performance $L(\theta)$ is given by the averaging cumulative rewards of episodes from the outer batch of every sampled task. Note that in the testing phase, no attacks are allowed. For a single meta policy, We repeat the test experiment with 10 random seeds and report mean rewards and standard deviations of these experimental results in \Cref{tab:num_res} and other tables in the following subsection.  
\subsection{Additional Experiment Results}\label{app:add_exp}
\paragraph{The intuition behind Sampling Attacks}
By scaling the reward signals by $\delta$, the attacker distorts the learner's perception of its reinforce loss, misleading the meta learning process. This distortion can be explained using the reinforce loss plots in \Cref{fig:compare}. We choose one representative experiment for each meta task under ISA and plot three kinds of reinforce loss functions (under smoothing) reviwed in the paragraph ``Meta Training''. The benchmark loss (red curves) is obtained by averaging all empirical reinforce loss (ERL) of policies adapted from the meta policy during the training phase in the benchmark setting (no attack). Denote by $\theta_t$ the meta policy at the $t$-th iteration in the benchmark training, the benchmark loss at time $t$ is given by  
\begin{align*}
	\text{Bechmark Loss} =&  \frac{1}{|\mathcal{I}|}  \sum_{i\in \mathcal{I}}\operatorname{ERL}(\theta_t^i),\quad  \theta_t^i = \theta_t+ \alpha \widehat{\nabla}J_i(\theta_t, \mathcal{D}_i^1)\\
	 \operatorname{ERL}(\theta_t^i):=&\frac{1}{|\mathcal{D}_i^2|}\sum_{\tau\in \mathcal{D}_i^2}[\sum_{h=1}^H \log \pi (a_h|s_h;\theta_t^i)A(s_h)], \quad (s_h,a_h)\in \tau,
\end{align*}        
where the inner batch $\mathcal{D}_i^i$ the outer batch $\mathcal{D}_i^2$ are sampled using $\theta_t$ and $\theta_t^i$, respectively (see \Cref{algo:sg-marl}). ERL is the MC estimation of the reinforce loss. 

The benchmark loss shows the training progress of meta RL in a neutral environment and serves as the baseline. To compare the training progress under adversarial attacks with this baseline, we introduce two other loss functions. Denote by $\hat{\theta}_t$ the meta policy obtained under ISA and by $\delta$ the attacker's manipulation, then the attacked loss (green curves for Intermittent Attack and blue curves for Persistent Attack) and the actual loss (yellow curves) are given by 
\begin{align*}
	\text{Attacked Loss} &= \frac{1}{|\mathcal{I}|}  \sum_{i\in \mathcal{I}}\operatorname{ERL}(\tilde{\theta}_t^i),\quad  \tilde{\theta}_t^i = \hat{\theta}_t+ \alpha\cdot \delta \widehat{\nabla}J_i(\theta_t, \mathcal{D}_i^1), \\
	\text{Actual Loss} & = \frac{1}{|\mathcal{I}|}  \sum_{i\in \mathcal{I}}\operatorname{ERL}(\hat{\theta}_t^i),\quad  \hat{\theta}_t^i = \hat{\theta}_t+ \alpha \widehat{\nabla}J_i(\theta_t, \mathcal{D}_i^1).
\end{align*} 
The attacked loss implies the learner's distorted perception of its training progress as its policy gradient is corrupted. The actual loss, on the other hand, reflects the true learning progress in the adversarial environment since the ERL of $\hat{\theta}_t^i$ reflects the current meta policy's adaptation performance under actual sample data.   

As shown in \Cref{fig:compare}, the agent mistakenly believes that it is making progress in the adversarial environment, as the green/blue curves go up (see Figures \ref{fig:compare}(a)(d)) or even above the benchmark (see Figures \ref{fig:compare}(b)). However, the actual performance of the learning either keeps deteriorating (see Figures \ref{fig:compare}(c)(d)) or remains at a modest level (see Figures \ref{fig:compare}(a)(b)).  
   
\begin{figure}
	\centering
	\subfigure[Reinforce loss under benchmark and Intermittent Attack in 2D navigation.]{\includegraphics[width=0.48\textwidth]{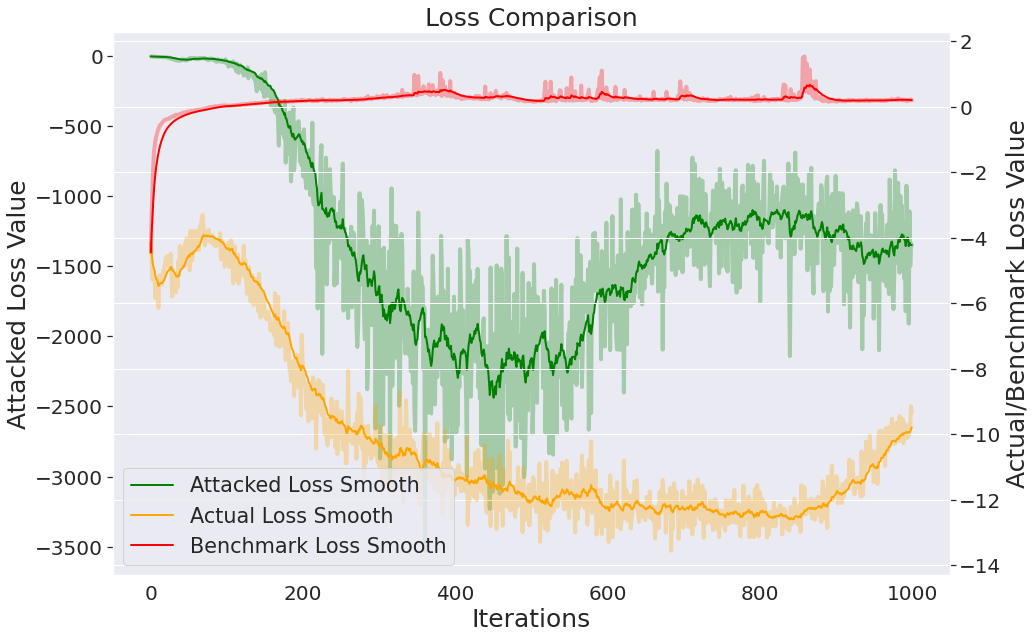}}
	\subfigure[Reinforce loss under benchmark and Persistent Attack in 2D navigation.]{\includegraphics[width=0.48\textwidth]{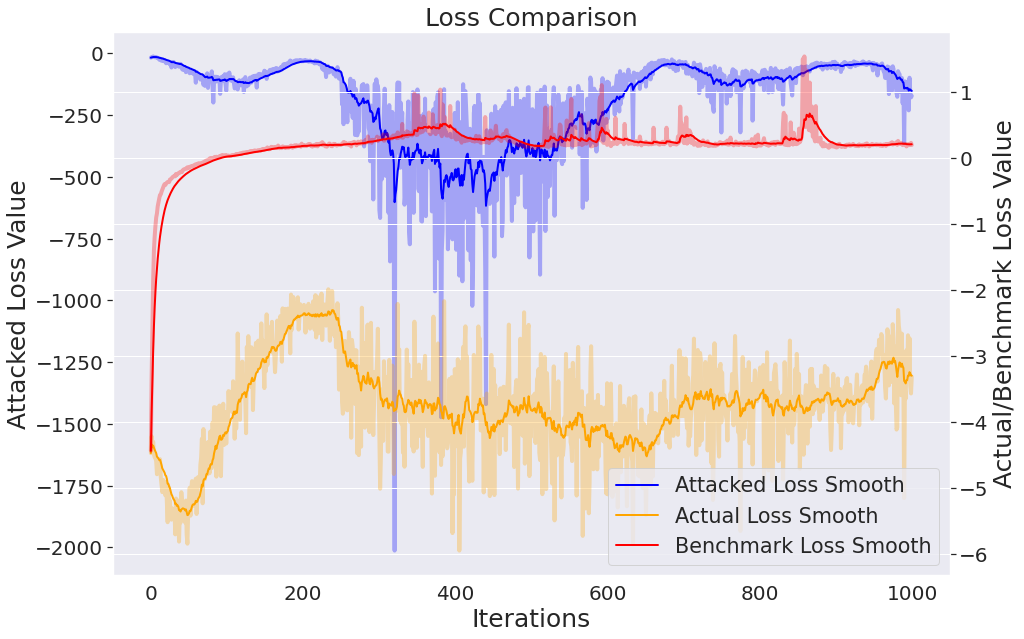}}
		\subfigure[Reinforce loss  under benchmark and Intermittent Attack in half-cheetah.]{\includegraphics[width=0.48\textwidth]{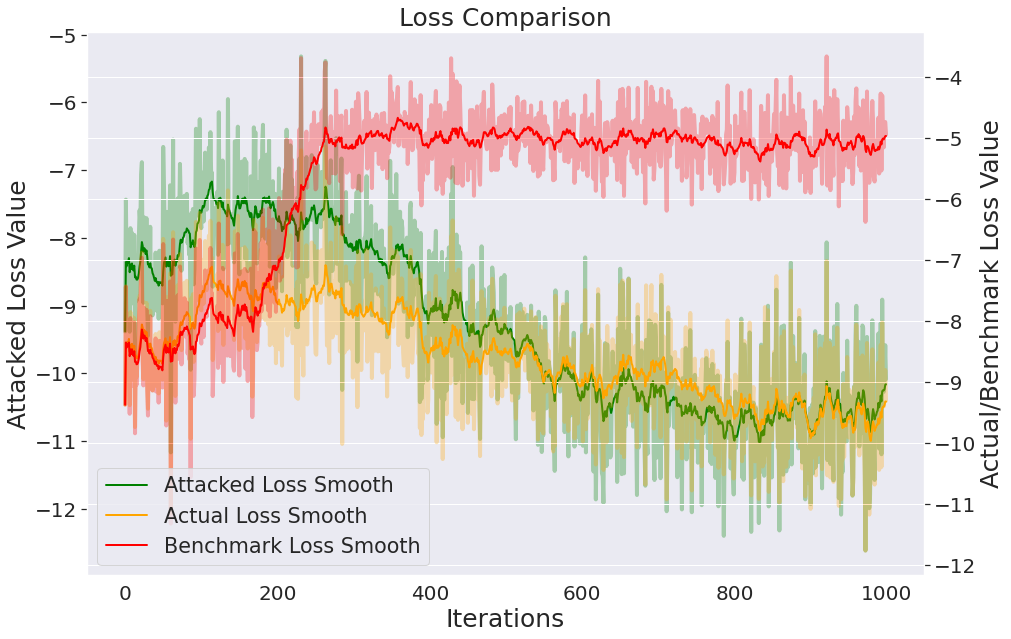}}
	\subfigure[Reinforce loss under benchmark and Persistent Attack in half-cheetah.]{\includegraphics[width=0.48\textwidth]{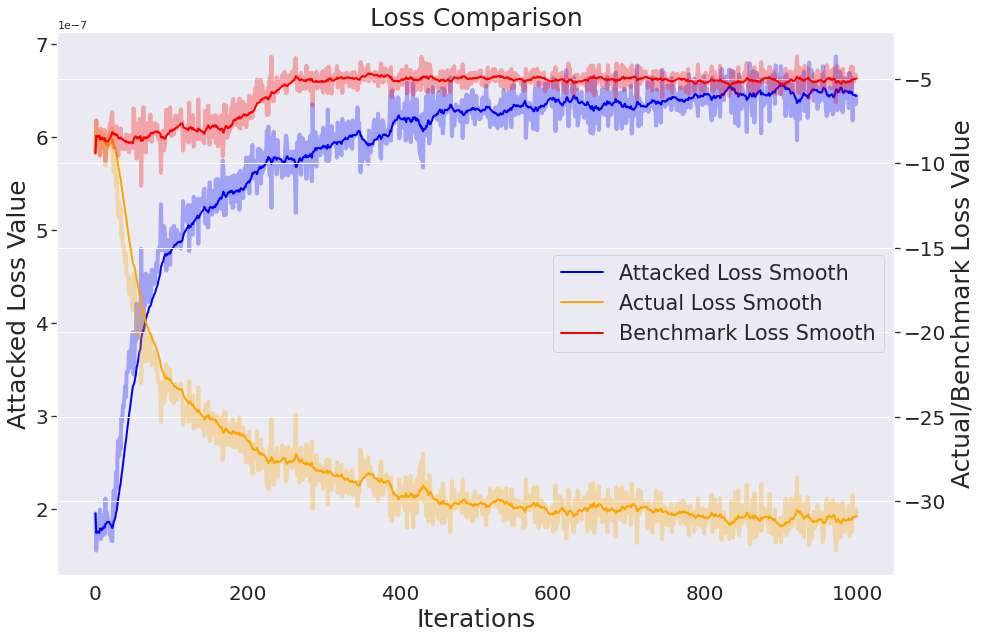}}
	\caption{Comparisons of the learner's empirical reinforce loss function under the benchmark and ISA settings.}
	\label{fig:compare}
\end{figure}

\paragraph{Additional Results on Intermittent and Persistent Attacks}
%TODO multistep, adv-value, and robust. 
This subsection presents additional experimental results on sampling attacks. We begin with experiments on attacking meta RL with two-step policy gradient adaptation (the one-step adaptation case is in \Cref{sec:pre} ).  \Cref{algo:sg-marl2} summarizes the MAMRL algorithm using two-step policy gradient adaptation. Note that when using two-step policy adaption, the inner sampling process requires two batches of trajectories $\mathcal{D}_i^1$ and $\mathcal{D}_i^2$, both of which are attacked in the experiments.
\begin{algorithm}[]	
	\caption{SG-MARL-2}
	\begin{algorithmic}[1]\label{algo:sg-marl2}
		\State \textbf{Input} Initialization $\theta_0$, $t=0$, step size $\alpha,\eta_1$.
		\While{not converge}
		\State Draw a batch of i.i.d tasks $\mathcal{T}_i, i\in \mathcal{I}$ with size $|\mathcal{I}|=I$;
		\For{$i\in\mathcal{I}$ }
		\State Sample a batch of trajectories $\mathcal{D}^1_i$ under $q_i(\cdot;\theta_t)$;
		\State $\bar{\theta}^i_{t+1}=\theta_t+\alpha \hat{\nabla} J_i(\theta_t,\mathcal{D}^1_i)$ ;
		\State Sample a batch of trajectories $\mathcal{D}_i^2$ under $q_i(\cdot;\bar{\theta}_{t+1}^i)$;
		\State $\theta^{i}_{t+1}=\bar{\theta}^{i}_{t+1}+\alpha \hat{\nabla} J_i(\bar{\theta}_{t+1}^i, \mathcal{D}_i^2) $ 
		\State Sample a batch of trjectories $\mathcal{D}_i^3$ under $q_i(\cdot;\theta_{t+1}^i)$;
		\EndFor
		\State Compute $\nabla_\theta^i MC$ using $\mathcal{D}_i^3$;
		\State $\theta_{t+1}=\theta_t+(\eta_1/I) \sum_{i\in\mathcal{I}}\nabla^i_{\theta_t}{MC}$;
		\EndWhile
		\State \textbf{Return } $\theta_{t+1}$
	\end{algorithmic}
\end{algorithm}

 Similar to our findings in the one-step case, the proposed sampling attack can significantly deteriorate the two-step adaptation performance of meta policies. Since Intermittent and Persistent Attacks are based on stochastic gradient descent, and the objective function is nonconvex-nonconcave, the two attack schemes do not return the same optimal attack $\delta^*$ or the same corrupted meta policy in general. \Cref{tab:optimal_attack} summarizes optimal attacks $\delta^*$ under different attack settings.  Finally, we comment on the robust training idea. As discussed in \Cref{sec:num}, decreasing the learner's step size turns the minimax into maxmin. For example, the maxmin problem in ISA is 
 \begin{align*}
 	\max_{\theta\in \R^d}\min_{\delta\in \Delta} \E_{i\sim p}\E_{\mathcal{D}\sim q_i(\cdot ;\theta)}[J_i(\theta+\alpha\cdot \delta\hat{\nabla}J(\theta,\mathcal{D}))].
 \end{align*} 
 In this case, the learner aims to maximize its training performance under sampling attacks, which does not necessarily lead to the maximization of the testing performance. It is not guaranteed that the meta policy obtained through the robust training would achieve satisfying adaptation performance in the testing phase. \Cref{tab:robust_train} summarizes the robust training experiments where $\eta_1=0.01, \eta_2=0.1$. Some encouraging results (shown in bold in \Cref{tab:robust_train}) do appear in some experiments where the obtained meta policies achieve comparable adaptation performance as the benchmark ones. However, the robust training does not succeed\footnote{We run multiple robust training experiments in half-cheetah, yet not every obtained meta policy shows adaptation performance comparable to the benchmark ones.} in the half-cheetah task under CSA (the blank entries in \Cref{tab:robust_train}).  This maxmin training remains largely a heuristic, and more efforts are needed to investigate defense mechanisms for meta learning thoroughly.  
\begin{table}
\centering
\caption{Experimental evaluations of Intermittent Attack and Persistent Attack under ISA and CSA when meta RL uses two-step gradient adaptation. Average cumulative rewards (and standard deviations) are  chosen as the evaluation metric, which is the sample mean of $L(\theta)$, $\theta$ being the meta policies corrupted by different attacks. Intermittent and Persistent Attacks significantly deteriorate the two-step adaption performance of meta policies. }
	\begin{tabular}{ccccc}
	\toprule 
	Task & Attack Model & Intermittent Attack & Persistent Attack & Benchmark\\
	\midrule
	\multirow{2}{*}{2D}& ISA& $-33.75\pm 6.24 $& $-362.07\pm 28.97 $ &\multirow{2}{*}{$-6.41\pm 0.62$}\\ 
	\cmidrule{2-4}
	& CSA &$-141.45\pm 3.96 $ & $-151.20\pm 36.60 $&\\
	\midrule
	\multirow{2}{*}{H-C}& ISA & $-91.92\pm 10.85 $&$ -127.17\pm 8.04 $ &\multirow{2}{*}{$-61.86 \pm 8.10 $}\\
	\cmidrule{2-4}
	& CSA & $-94.11\pm 10.82 $& $-104.53\pm 9.46 $ &\\
	\bottomrule
	\end{tabular}
	\label{tab:num_step2} 
\end{table}

\begin{table}
	\centering
	\caption{A summary of optimal attacks $\delta^*$ under different attack settings.}
	\label{tab:optimal_attack}
	\begin{tabular}{ccccc}
	\toprule
		Task & gradient step & Attacl Model & Intermittent Attack & Persistent Attack\\ 
		\midrule
		\multirow{4}{*}{2D} & \multirow{2}{*}{1} & ISA & -0.24 & -39.56\\
		& & CSA & -12.96 & -38.58\\
		\cmidrule{2-5}
		&\multirow{2}{*}{2} & ISA & 2.78 & -2.36\\
		& & CSA & -0.29 & -2.60\\
		\midrule
		\multirow{4}{*}{H-C} & \multirow{2}{*}{1} & ISA & 1.11 & 0.98\\
		& & CSA & -5.94 & 0.89\\
		\cmidrule{2-5}
		&\multirow{2}{*}{2} & ISA & 5.81 & 0.06\\
		& & CSA &4.23 & 0.66\\
		\bottomrule 
	\end{tabular}
\end{table}

\begin{table}
	\centering
	\caption{A summary of adaption performance of robust meta policy against Persistent Attack.}
	\label{tab:robust_train}
	\begin{tabular}{cccccc}
	\toprule
		Task & gradient step & Attacl Model & Persistent Attack & Robust Training& Benchmark \\ 
		\midrule
		\multirow{4}{*}{2D} & \multirow{2}{*}{1} & ISA & $-298.77\pm 34.09$ & $\mathbf{-11.50\pm 0.57}$ & \multirow{2}{*}{$-10.92\pm 0.85$}\\
		& & CSA & $-488.97\pm 31.06$ &$\mathbf{-11.34\pm 0.48}$ \\
		\cmidrule{2-6}
		&\multirow{2}{*}{2} & ISA & $-362.07\pm 28.97$ & $-33.67\pm 1.34$& \multirow{2}{*}{$-6.41\pm 0.62$}\\
		& & CSA & $-151.20\pm 36.60$  & $-26.26\pm 1.41$&\\
		\midrule
		\multirow{4}{*}{H-C} & \multirow{2}{*}{1} & ISA & $-101.13\pm 10.57$ & $\mathbf{-63.26\pm 8.69}$& \multirow{2}{*}{$-63.41\pm 6.86$}\\
		& & CSA & $-93.65\pm 9.81$ & $---$\\
		\cmidrule{2-6}
		&\multirow{2}{*}{2} & ISA & $-127.17\pm 8.04$ & $-75.09\pm 9.29$& \multirow{2}{*}{$-61.04\pm 8.10$}\\
		& & CSA &$-104.53\pm 9.46$ & $---$&\\
		\bottomrule 
	\end{tabular}
\end{table}

\section{Regularity Assumptions}\label{app:regular}
In this section, we provide further justifications for the regularity conditions assumed in this paper.  In our analysis, it is assumed that the policy gradient $\nabla J_i(\theta)$ is bounded as stated in \Cref{ass:bound_pg}. 
\begin{assumption}[Bounded Policy Gradient]\label{ass:bound_pg}
	There exists a constant $G$, such that for any trajectory $\tau$ and policy parameter $\theta$, $\|g_i(\tau;\theta)\|\leq G$. Equivalently, for any $\theta$ and any batch of trajectories $\mathcal{D}$, $\|\nabla J_i(\theta)\|\leq G$, $\|\hat{\nabla} J_i(\theta,\mathcal{D})\|\leq G$.  
\end{assumption}

Moreover, the value function $f(\delta,\theta)$ is assumed to Lipschitz smooth in \Cref{ass:lip}.
\begin{assumption}[Lipschitz Gradients]\label{ass:lip}
	The function $f$ is continuously differentiable in both $\delta$ and $\theta$. Furthermore, there exists constants $L_{11}, L_{12}$ and $L_{22}$ such that for every $\delta, \delta_1,\delta_2$ and $\theta,\theta_1,\theta_2$, 
	\begin{align}\label{eq:grad_lip}
	\begin{aligned}
		&\|\nabla_\delta f(\delta_1,\theta)-\nabla_\delta f(\delta_2,\theta)\|\leq L_{11}\|\delta_1-\delta_2\|,\\
		 &\|\nabla_\theta f(\delta,\theta_1)-\nabla_\theta f(\delta,\theta_2)\|\leq L_{22}\|\theta_1-\theta_2\|,\\
		&\|\nabla_\delta f(\delta,\theta_1)-\nabla_\delta f(\delta,\theta_2)\|\leq L_{12}\|\theta_1-\theta_2\|,\\
		&\|\nabla_\theta f(\delta_1,\theta)-\nabla_\theta f(\delta_2,\theta)\|\leq L_{12}\|\delta_1-\delta_2\|. 
	\end{aligned}  
	\end{align}
\end{assumption} 

The two assumptions are direct extensions of \citep[Lemma~1]{fallah_sgmrl} in the context of adversarial MAMRL, which is established on customary assumptions (e.g., bounded rewards) in the policy gradient literature \cite{kakde_pg}. On the other hand, the unbiasedness and finite variance of the stochastic gradient assumed in \Cref{ass:sto} is also valid in MAMRL, which has been proved in \cite{fallah_sgmrl}, in order to show the convergence of \Cref{algo:sg-marl}.    
\begin{assumption}[Stochastic Gradients]\label{ass:sto}
	The stochastic gradient $\hat{\nabla}f(\delta,\theta,\xi)$, with $\xi$ being the random variable corresponding to the batch of trajectories with batch size $M$, satisfies
	\begin{align}\label{eq:grad_sto}
			\begin{aligned}
		&\mathbb{E}[\hat{\nabla}f(\delta,\theta,\xi)-\nabla f(\delta,\theta)]=0\\
		&\mathbb{E}[\|\hat{\nabla}f(\delta,\theta,\xi)-\nabla f(\delta,\theta)\|^2]\leq \sigma^2/M.
	\end{aligned}
	\end{align} 
\end{assumption}

Independent of the above assumptions regarding the regularity of the value function $f(\delta,\theta)$, we require that the admissible set of attacks $\delta$ is also regular as stated in \Cref{ass:domain}, which helps the derivation of complexity results.
\begin{assumption}\label{ass:domain}
	The set $\Delta$ is convex and compact, with its diameter upper bounded by $D\in \R_{+}$. Without loss of generality, it is assumed that $D\geq 1$. 
\end{assumption}

In addition to the assumptions discussed above, PL condition in \Cref{ass:minmaxpl} and RSI condition in \Cref{ass:rsi} play an important role in showing the convergence of the proposed attack mechanisms. Given their relationships with respect to strong convexity(SC) \cite{karimi16plqc}:
\begin{align*}
	\mbox{SC}\Rightarrow \mbox{RSI} \Rightarrow \mbox{PL},
\end{align*}    
this work extends the convergence results of Gradient Descent Ascent \citep{karimi16plqc,lin_twotime} to nonconvex situations under weaker conditions. The key point is that the linear contraction of $\E[\|\theta^*_t-\theta_t\|^2]$ (see \Cref{lemma:decreasing_error}), previously established under strong convexity \cite{lin_twotime}, also holds true under \Cref{ass:rsi}.   Our theoretical treatment on nonconvex-nonconcave minimax problems can be helpful for future studies on this topic. Even though it is unclear under which condition, the value function of RL or meta RL will  satisfy PL or RSI, we provide numerical evidence in \Cref{fig:verify_pl} and \Cref{fig:verify_rsi}, demonstrating their validity in the experiments. 
\begin{figure}[!ht]
	\centering
	\subfigure[The comparison between $\|\nabla_\theta f(\delta_t,\theta_t)\|$ and $(\max_{\theta}f(\delta_t,\theta)-f(\delta_t,\theta_t))$ in 2D navigation]{\includegraphics[width=0.49\textwidth]{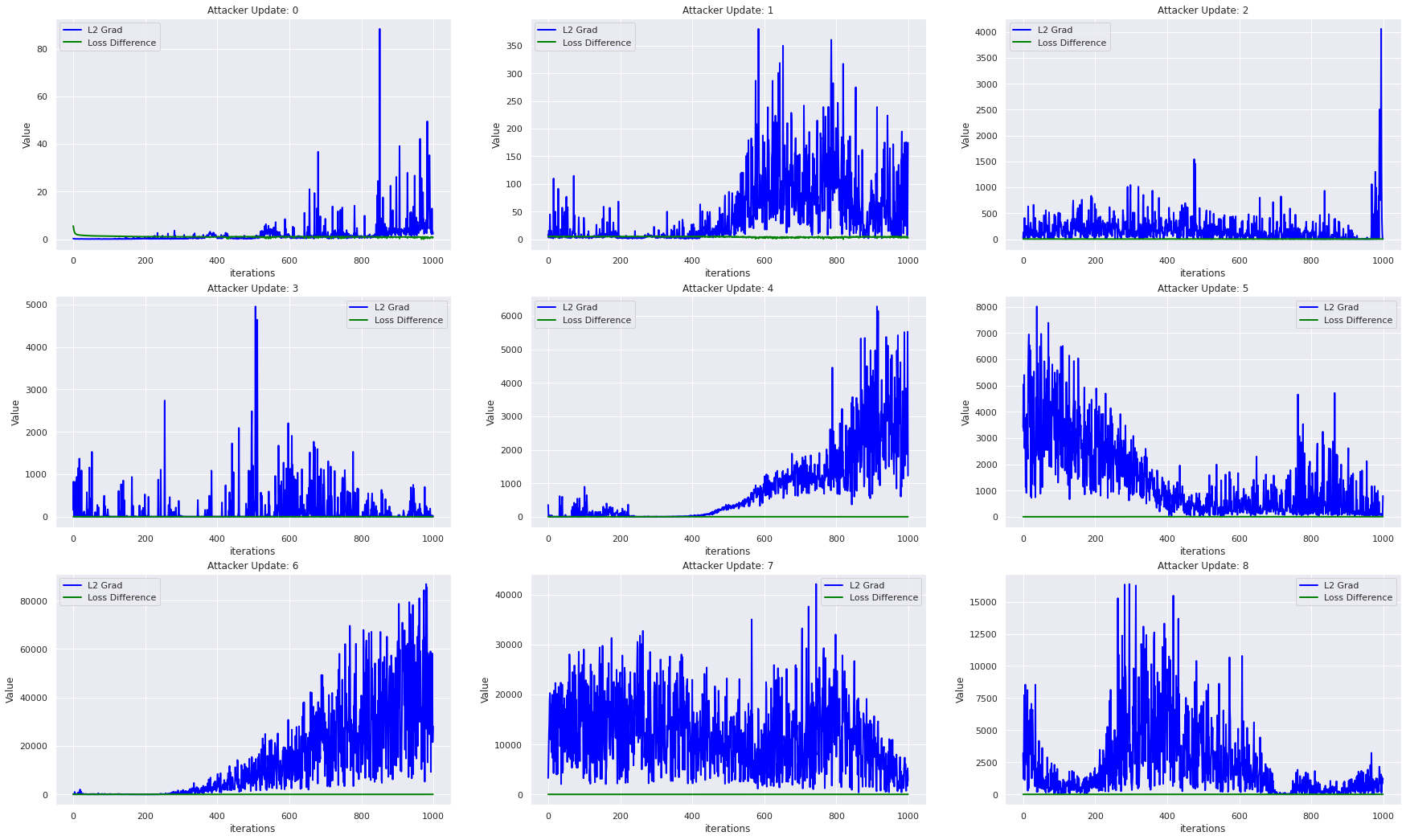}}
	\subfigure[The comparison between $\|\nabla_\theta f(\delta_t,\theta_t)\|$ and $(\max_{\theta}f(\delta_t,\theta)-f(\delta_t,\theta_t))$ in Half-cheetah]{\includegraphics[width=0.49\textwidth]{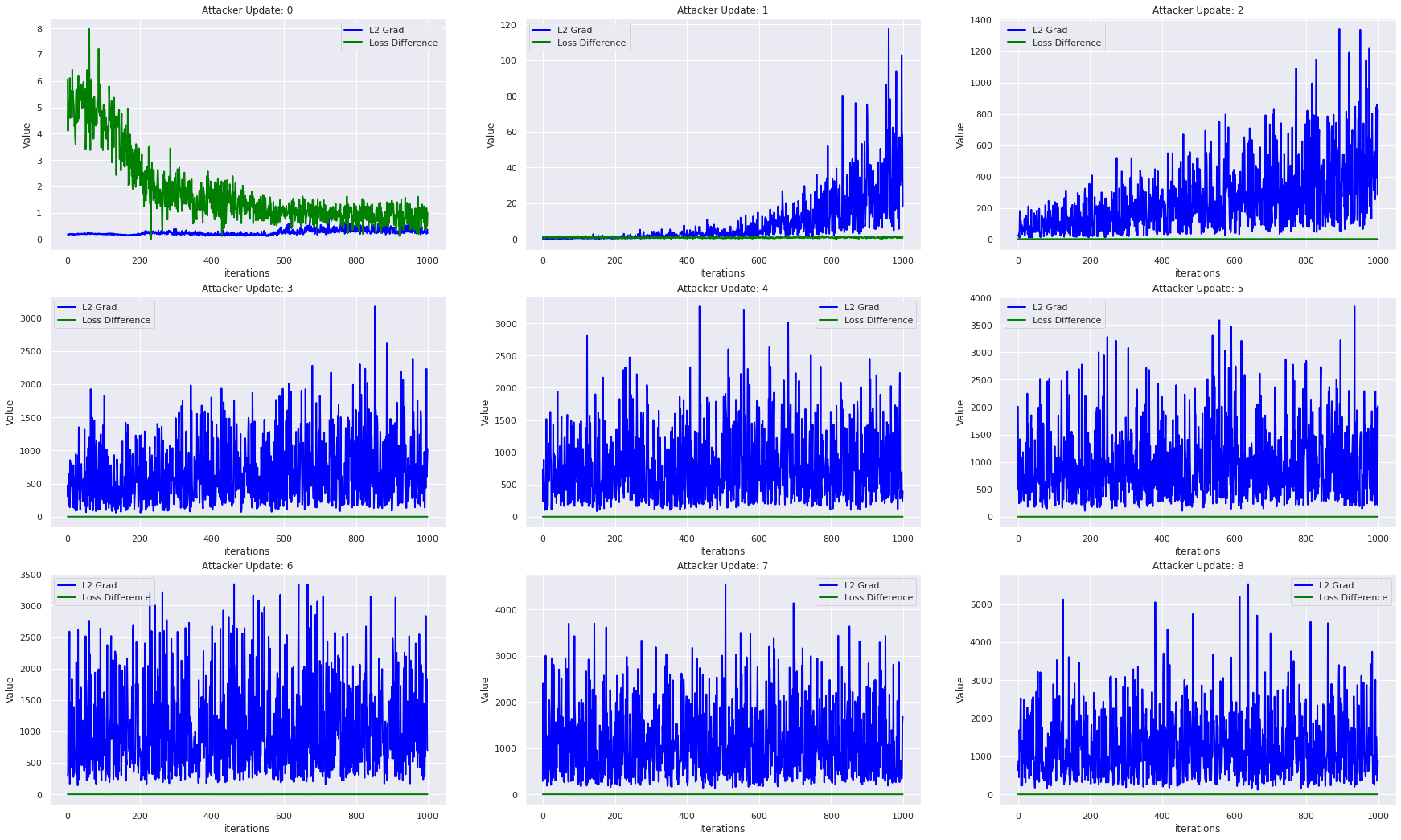}}
	\caption{Numerical Justifications of \Cref{ass:minmaxpl}. As the meta learning proceeds, the norm of gradients (blue curves) are above the loss difference (green curves). }
	\label{fig:verify_pl}
\end{figure}
\begin{figure}[!ht]
	\centering
	\subfigure[The comparison between $\left\langle\nabla_\theta f(\delta_t,\theta_t), \theta^*(\delta)-\theta_t\right\rangle$ and $\|\theta^*-\theta_t\|^2$ in 2D navigation]{\includegraphics[width=0.49\textwidth]{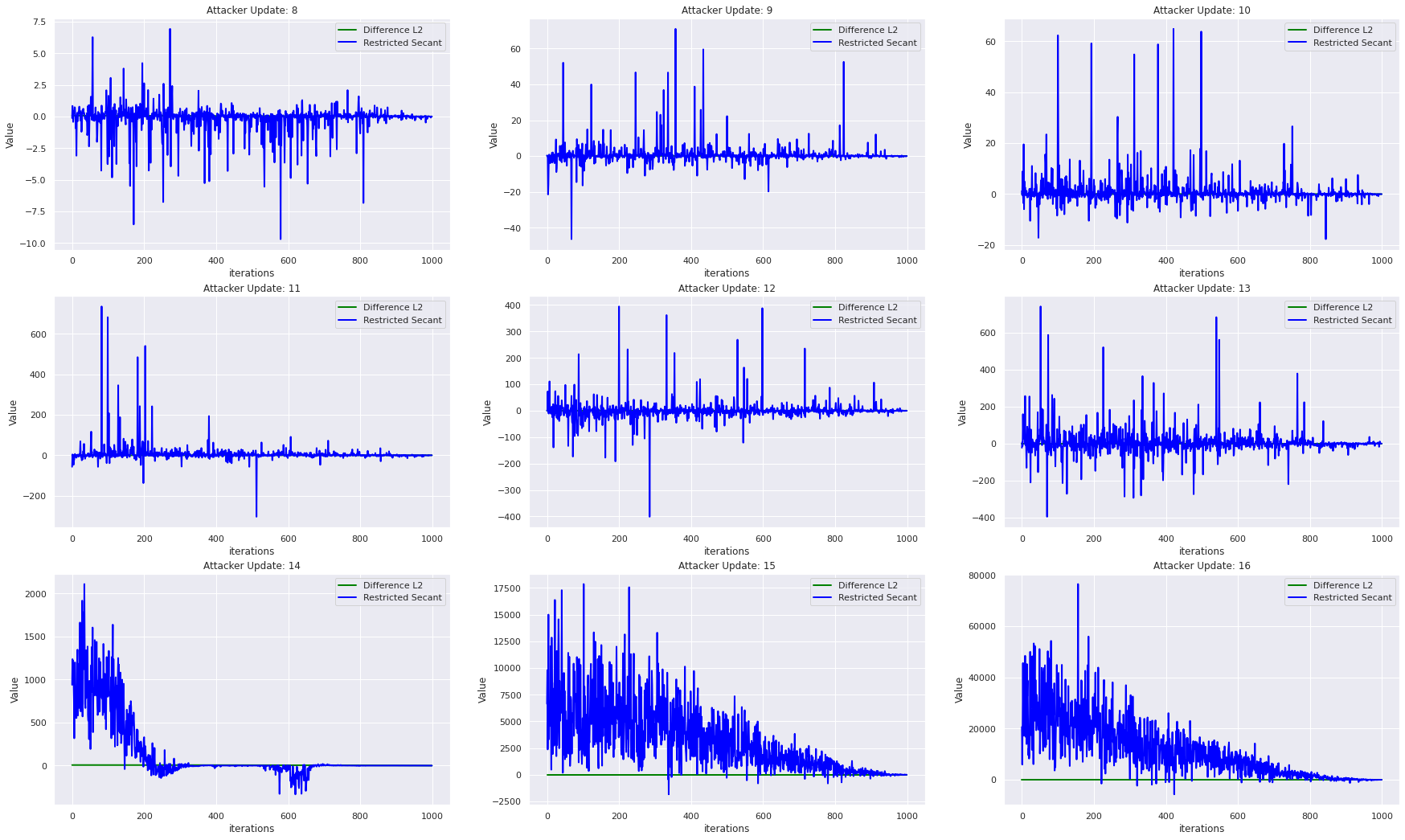}}
	\subfigure[The comparison between $\left\langle\nabla_\theta f(\delta_t,\theta_t), \theta^*(\delta)-\theta_t\right\rangle$ and $\|\theta^*-\theta_t\|^2$ in Half-cheetah]{\includegraphics[width=0.49\textwidth]{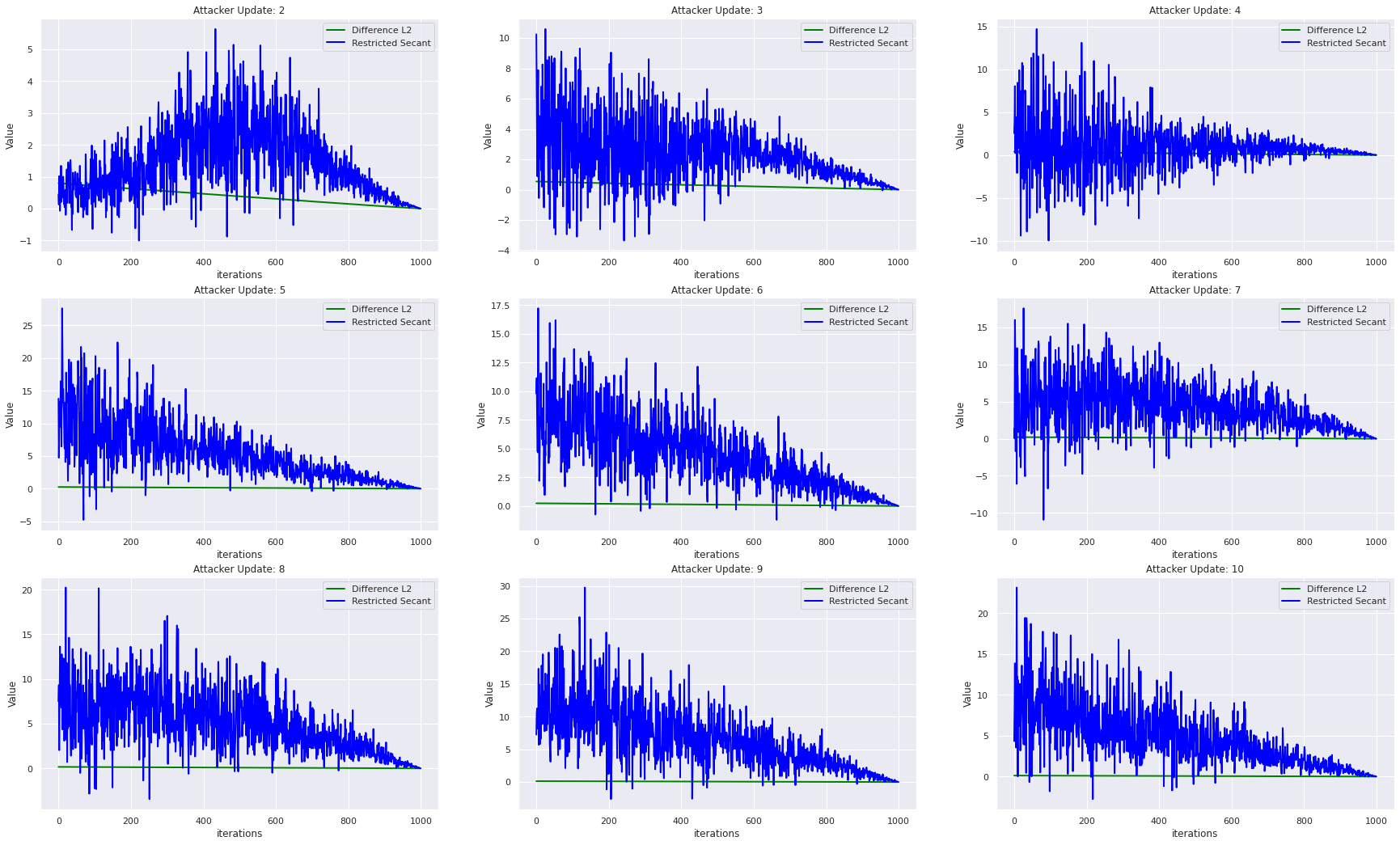}}
	\caption{Numerical Justifications of \Cref{ass:minmaxpl}. As the meta learning proceeds, the restricted secant $\left\langle\nabla_\theta f(\delta_t,\theta_t), \theta^*(\delta)-\theta_t\right\rangle$ (blue curves) are above the $\ell_2$ difference (green curves) up to a constant. }
	\label{fig:verify_rsi}
\end{figure}

\section{Proofs}\label{app:proof}
\setcounter{equation}{0}
\renewcommand{\theequation}{\Alph{section}\arabic{equation}}
\setcounter{lemma}{0}
\renewcommand{\thelemma}{\Alph{section}\arabic{lemma}}
\setcounter{theorem}{0}
\renewcommand{\thetheorem}{\Alph{section}\arabic{theorem}}

%TODO CSA extension
\subsection{Proofs in \Cref{sec:model}}\label{app:mmp_up}
\begin{proof}[Proof of \Cref{lem:upper-bound}]
	As shown in \citep[lemma 1]{fallah_sgmrl}, under \Cref{ass:bound_pg}, for any $\theta_1,\theta_2\in \R^d$, we have 
	\begin{align*}
		|J_i(\theta_1)-J_i(\theta_2)|\leq C(G)\|\theta_1-\theta_2\|, 
	\end{align*}
	where $C(G)$ is a constant depending on $G$. Fix a $\delta\in \Delta$, and let $\theta_1=\theta^*+\alpha \widehat{\nabla} J_i(\theta^*,\mathcal{D})$, $\theta_2=\theta^*+\alpha\cdot \delta \widehat{\nabla}J_i(\theta^*,\mathcal{D})$, the above inequality leads to 
	\begin{align*}
		|J_i(\theta^*+\alpha \widehat{\nabla} J_i(\theta^*,\mathcal{D}))-J_i(\theta^*+\alpha\cdot \delta \widehat{\nabla}J_i(\theta^*,\mathcal{D}))|&\leq C(G)\alpha\|\delta-1\|\|\widehat{\nabla}J_i(\theta^*,\mathcal{D})\|\\
		&\leq C(\alpha,G)\|\delta-1\|,
	\end{align*}
	where the second inequality holds by the boundedness of $\|\widehat{\nabla}J_i(\theta^*,\mathcal{D})\|$ assumed in \Cref{ass:bound_pg}. Combining triangle inequality and the above one yields \eqref{eq:pre_upper}. Taking expectations on both sides of \eqref{eq:pre_upper}, and further taking the minimization of the left-hand side, we obtain 
	\begin{align*}
		\min_{\delta\in \Delta}\E_{i\sim p}\E_{\mathcal{D}\sim q_i(\cdot;\theta)}[ J_i(\theta^*+\alpha \hat{\nabla}J_i(\theta^*, \mathcal{D}))]&\leq  \E_{i\sim p}\E_{\mathcal{D}\sim q_i(\cdot;\theta)}[ J_i(\theta^*+\alpha \cdot \delta \hat{\nabla}J(\theta^*, {\mathcal{D}}))]\\
		&\qquad + C(\alpha,G)\|\delta-1\|, \qquad \forall \delta.
	\end{align*} 
	Finally, taking the minimization of the right-hand side of the above inequality leads to \eqref{eq:stacktominimax}.
\end{proof}

%TODO extension to CSA: prove a lemma and say something about the theorem; briefly introduce the algorithm
\paragraph{Extensions to CSA}

\begin{proposition}
	Under \Cref{ass:bound_pg} and \Cref{ass:domain}, there exists a constant $C$ such that for $\theta^*=\theta^*(\delta)$ defined in \eqref{eq:com-attack} and any batch of trajectories $\mathcal{D}$,
	\begin{align}
		J_i(\theta^*+\alpha \widehat{\nabla} J_i(\theta^*, \mathcal{D}))\leq \delta J_i(\theta^*+\alpha\cdot \delta \widehat{\nabla}J_i(\theta^*,\mathcal{D}))+ C|\delta-1|.
		\label{eq:csa_ineq}
	\end{align}
	Hence, 
	\begin{align}
		&\min_{\delta\in \Delta}\E_{i\sim p}\E_{\mathcal{D}\sim q_i(\cdot;\theta)}[ J_i(\theta^*+\alpha \hat{\nabla}J_i(\theta^*, \mathcal{D}))]\\
		&\leq \min_{\delta\in \Delta} \E_{i\sim p}\E_{\mathcal{D}\sim q_i(\cdot;\theta)}[ \delta J_i(\theta^*+\alpha \cdot \delta \hat{\nabla}J(\theta^*, {\mathcal{D}}))]+ C|\delta-1|.
		\label{eq:csa_expect}
	\end{align}
\end{proposition}
\begin{proof}
	Note that 
	\begin{align*}
		&|J_i(\theta^*+\alpha \hat{\nabla}J_i(\theta^*, \mathcal{D}))-\delta J_i(\theta^*+\alpha\cdot \delta \widehat{\nabla}J_i(\theta^*,\mathcal{D}))|\\
		&\leq |J_i(\theta^*+\alpha \hat{\nabla}J_i(\theta^*, \mathcal{D}))-\delta J_i(\theta^*+\alpha\widehat{\nabla}J_i(\theta^*,\mathcal{D}))|\\
		&\qquad  +|\delta J_i(\theta^*+\alpha\widehat{\nabla}J_i(\theta^*,\mathcal{D}))-\delta J_i(\theta^*+\alpha\cdot \delta \widehat{\nabla}J_i(\theta^*,\mathcal{D}))|\\
		& \leq |\delta-1||J_i(\theta^*+\alpha\widehat{\nabla}J_i(\theta^*,\mathcal{D}))|+|\delta|C(\alpha, G) |\delta-1|,
	\end{align*}
	where in the last inequality we use the following result obtained in the proof of \Cref{lem:upper-bound} 
	\begin{align*}
		|J_i(\theta^*+\alpha \widehat{\nabla} J_i(\theta^*,\mathcal{D}))-J_i(\theta^*+\alpha\cdot \delta \widehat{\nabla}J_i(\theta^*,\mathcal{D}))|\leq C(\alpha,G)\|\delta-1\|.
	\end{align*}
	
	Let $\|r\|_\infty:=\max_{s,a}|r(s,a)|$ (assumed to be bounded). According to the definition of the value function $J_i$, the horizon length $H$, and the discounting factor $\gamma$, we obtain  $$|J_i(\theta^*+\alpha\widehat{\nabla}J_i(\theta^*,\mathcal{D}))|\leq \frac{1-\gamma^H}{1-\gamma}\|r\|_\infty. $$ Since $\delta\in \Delta$, $|\delta|$ can be controlled by the diameter $D$. Hence, we obtain
	\begin{align*}
		|J_i(\theta^*+\alpha \hat{\nabla}J_i(\theta^*, \mathcal{D}))-\delta J_i(\theta^*+\alpha\cdot \delta \widehat{\nabla}J_i(\theta^*,\mathcal{D}))|\leq C(\alpha, G, \|r\|_\infty, D) |\delta-1|,
	\end{align*}
	leading to \eqref{eq:csa_ineq}. Following the same argument in the proof of \Cref{lem:upper-bound}, we obtain \eqref{eq:csa_expect}.
	\end{proof}

The above proposition leads to the MMP formulation of CSA in below.
\begin{align*}
	\min_{\delta\in \Delta}\max_{\theta\in \R^d}\E_{i\sim p}\E_{\mathcal{D}\sim q_i(\cdot;\theta)}[ \delta J_i(\theta+\alpha \cdot \delta \hat{\nabla}J(\theta, {\mathcal{D}}))].
\end{align*}
Since $\Delta$ is a bounded compact convex set, our nonasymptotic convergence analysis on Intermittent and Persistent Attacks under ISA also applies to CSA using the above MMP formulation. Finally, we comment on the relationship between MMP and ISA(CSA). The original Stackelberg formations in \eqref{eq:inner-attack} and \eqref{eq:com-attack} target the testing performance, while in \eqref{eq:minimax_attack}, the objective is the training performance. Hence, the two are different problems, even though the minimax values serves as the upper bounds for \eqref{eq:inner-attack} and \eqref{eq:com-attack}. When the access to the testing performance is not available during the training phase, this minimax upper bound helps guide the attacker's update.

\subsection{Proofs in \Cref{sec:analysis_intermittent}}\label{app:inter_proof}

In this section, we present the full version of \Cref{thm:intermittent} with detailed choices of step sizes $\eta_1,\eta_2$ and batch size $M$. We begin with some technical lemmas.  

%\begin{lemma}\label{lemma:phi_diff}
%	Under \Cref{ass:minmaxpl} and other regularity assumptions, there exists $\theta^*\in \argmax_{\theta}f(\delta,\theta)$, such that $
%		\nabla \Phi(\delta)=\nabla_{\delta} f(\delta,\theta^*)=\E[\hat{\nabla}f(\delta,\theta^*;\xi)]$.
%	Moreover, there exist a constant $L$, such that $\|\nabla\Phi(\delta_1)-\nabla\Phi(\delta_2)\|\leq L\|\delta_2-\delta_1\|$. 
%\end{lemma}
%The above lemma states that the gradient of $\max_\theta f(\delta,\theta)$ can be directly evaluated using the gradient $\nabla_\delta f(\delta,\theta^*)$. Moreover, such gradient function $\nabla \Phi$ is Lipschitz smooth, reducing the attacker's convergence problem to a standard analysis on gradient descent of nonconvex objectives. The main result regarding Intermittent Attack is stated as follows.

\begin{lemma}\label{lemma:phi_diff}
Under \Cref{ass:minmaxpl}, \Cref{ass:lip} and \Cref{ass:sto},  there exists a $\theta^*\in \argmax_{\theta}f(\delta,\theta)$, such that
	\begin{align*}
		\nabla \Phi(\delta)=\nabla_{\delta} f(\delta,\theta^*)=\E[\hat{\nabla}f(\delta,\theta^*;\xi)].
	\end{align*}
	Moreover, $\Phi(\delta)$ is $L$-Lipschitz smooth with the Lipschitz constant $L:=L_{11}+\frac{L^2_{12}}{2\mu}$,
	\begin{align*}
		\|\nabla\Phi(\delta_1)-\nabla\Phi(\delta_2)\|\leq L\|\delta_2-\delta_1\|. 
	\end{align*}	
\end{lemma}  
\begin{proof}[Proof of \Cref{lemma:phi_diff}]
	Our first step is to show that for any $\delta_1,\delta_2$ and $\theta_1\in \argmax_{\theta}f(\delta_1,\theta)$, there exists $\theta_2\in \argmax_{\theta}f(\delta_,\theta)$ such that $\|\theta_1-\theta_2\|\leq \frac{L_{12}}{2\mu}\|\delta_1-\delta_2\|$.
	
	First, based on the assumption of Lipschitz gradients in \eqref{eq:grad_lip}, 
	\begin{align*}
		\|\nabla_{\theta} f(\delta_2,\theta_1)-\nabla_{\theta}f(\delta_1,\theta_1)\|\leq L_{12}\|\delta_2-\delta_1\|,
	\end{align*}
which is equivalent to $\|\nabla_{\theta} f(\delta_2,\theta_1)\|\leq L_{12}\|\delta_2-\delta_1\|$, since $\nabla_{\theta}f(\delta_1,\theta_1)=0$. Then, using the PL condition in \eqref{ass:minmaxpl}, we obtain 
\begin{align*}
	\max_\theta f(\delta_2,\theta)-f(\delta_2,\theta_1)\leq \frac{1}{2 \mu} \|\nabla_\theta f(\delta_2,\theta_1)\|^2,
\end{align*}
which implies that 
\begin{align}
	\Psi_{\delta_2}(\theta_1)-\min_{\theta} \Psi_{\delta_2}(\theta)\leq \frac{L_{12}^2}{2 \mu}\|\delta_2-\delta_1\|^2.\label{eq:leq}
\end{align}

	As shown in \citep[Theorem 2]{karimi16plqc}, PL condition implies the quadratic growth, that is, if $\Psi_\delta(\theta)$ satisfies $\mu$-PL condition, the following holds true 
	\begin{align}\label{eq:geq}
		\Psi_{\delta}(\theta)-\min_{\theta}\Psi_{\delta}(\theta)\geq 2\mu\|\theta-\theta^*\|^2,
	\end{align} 
	where $\theta^*\in \argmin_{\theta} \Psi_\delta(\theta) $ and $\|\theta-\theta^*\|=\min_{\theta'\in \argmin_{\theta} \Psi_\delta(\theta)}\|\theta-\theta'\|$. In other words, $\theta^*$ is the projection of $\theta$ onto the set of optimal solution, and $\|\theta-\theta^*\|$ is the distance of the point $\theta$ to the optimal solution set.
		
	Finally,  let $\theta_2$ be the projection of $\theta_1$ onto the set $\{\theta'|\theta'\in argmax_\theta f(\delta_2,\theta)\}$, then combining \eqref{eq:leq} and \eqref{eq:geq} leads to the desired result
	\begin{align}\label{eq:theta_delta}
		\|\theta_1-\theta_2\|\leq \frac{L_{12}}{2\mu}\|\delta_1-\delta_2\|.
	\end{align}
	
	The second step is to show the existence of $\nabla\Phi(\delta)$ and $\nabla \Phi(\delta)=\nabla_\delta f(\delta,\theta^*)$, which is straightforward from \citep[lemma A.5]{gda_two}. Finally, we show that $\Phi(\delta)$ is also Lipschitz smooth because of \Cref{ass:lip}.  Notice that 
	\begin{align*}
		\|\nabla\Phi(\delta_1)-\nabla\Phi(\delta_2)\|&=\|\nabla_\delta f(\delta_1,\theta_1)-\nabla_\delta f(\delta_2,\theta_2)\|\\
			&=\|\nabla_\theta f(\delta_1,\theta_1)-\nabla_\delta f(\delta_2,\theta_1)+\nabla_\delta f(\delta_2,\theta_1)-\nabla_\theta f(\delta_2,\theta_2)\|\\
			&\leq L_{11}\|\delta_1-\delta_2\|+L_{12}\|\theta_1-\theta_2\|,
	\end{align*} 
	together with \eqref{eq:theta_delta}, we have 
	\begin{align*}
		\|\nabla\Phi(\delta_1)-\nabla\Phi(\delta_2)\|\leq (L_{11}+\frac{L^2_{12}}{2\mu}) \|\delta_1-\delta_2\|,
	\end{align*}
	showing that $\Phi(\delta)$ is also Lipschitz smooth with the Lipschitz constant being $L_{11}+\frac{L^2_{12}}{2\mu}$. 
\end{proof}
The above lemma, as an extension of \citep[Lemma A.1]{gda_two},  serves as a substitute of Danskin's theorem in convex analysis \citep{BERNHARD19951163}, which states that the gradient of $\max_\theta f(\delta,\theta)$ can be directly evaluated using the gradient $\nabla_\delta f(\delta,\theta^*)$. Moreover, such gradient function $\nabla \Phi$ is Lipschitz smooth, reducing the attacker's convergence problem to a standard analysis of gradient descent of nonconvex objectives.

Our next lemma shows that $\theta_{K}$ returned by the inner loop properly approximates $\theta^*(\delta)$ in the sense that $\nabla_\delta f(\delta_t,\theta^K_t)\approx \nabla\Phi(\delta_t)$, where $\theta^K_t:=\theta_K(\delta_t)$ denotes the final iterate of the inner loop under $\delta_t$. The proof of the following lemma relies on the assumption that there exists a constant $B$, such that $\Phi(\delta_t)-f(\delta_t,\theta_0)\leq B$ holds for every $t$, which can be verified using 
 \Cref{lemma:phi_diff}, as shown in \citep[Lemma~A.6]{gda_two}
 \begin{lemma}\label{lem:zt}
	Let $\rho=1-\frac{2\mu}{L_{22}}$, $\bar{L}=\max\{L_{12}, L_{22}, 1, \Phi_\infty\}$, where $\Phi_\infty:=\max\{\max_\delta\|\nabla \Phi\|,1\}$. Assume that there exists a constant $B>0$ such that $\Phi(\delta_t)-f(\delta_t,\theta_0)\leq B$. For any $\epsilon>0$, if $K$ and $M$ are large enough such that 
	\begin{align*}
	&K\geq 	\frac{1}{\log \rho^{-1}}\left(4\log\epsilon^{-1}+\log \frac{2^4 D^2(2\Phi_\infty+LD)^2\bar{L}^2B}{L^2 \mu}\right)\\
	& M\geq \frac{24\sigma^2D^2(2\Phi_\infty+LD)^4}{2\bar{L}L^2\epsilon^4}
	\end{align*}
     then for $z_t:=\nabla_\delta f(\delta_t,\theta^K_t)-\nabla\Phi(\delta_t)$, 
     \begin{align}\label{eq:zt_ineq}
     \E[\|z_t\|]\leq \frac{L\epsilon^2}{4D (2\Phi_\infty+LD)^2}	
     \end{align}
      and 
      \begin{align}\label{eq:theta_fosp}
      \E[\|\nabla_\theta f(\delta_t,\theta^K_t)\|]\leq \epsilon.	
      \end{align}                     
\end{lemma}

\begin{proof}[Proof of \Cref{lem:zt}]
Since $f(\delta,\theta)$ is assumed to be Lipschitz-smooth, $-f(\delta,\theta)$ is also Lipschitz-smooth, and we have 
\begin{align*}
	-f(\delta_t,\theta^K_t)
	\leq -f(\delta_t,\theta_{t-1}^K)-\left\langle\nabla_\theta f(\delta_t,\theta_t^{K-1}), \theta_t^K-\theta_{t}^{K-1} \right\rangle +\frac{L_{22}}{2}\|\theta_t^K-\theta_{t}^{K-1}\|^2.
\end{align*}
Plugging the updating rule into the above inequality, we obtain
\begin{align*}
	-f(\delta_t,\theta^K_t)
	\leq -f(\delta_t,\theta_{t-1}^K)-\frac{1}{L_{22}}\left\langle \nabla_\theta f(\delta_t,\theta_t^{K-1}), \widehat{\nabla}_\theta f(\delta_t,\theta_t^{K-1},\xi_t^{K-1})\right\rangle+\frac{1}{2L_{22}}\|\widehat{\nabla}_\theta f(\delta_t,\theta_t^{K-1},\xi_t^{K-1})\|^2,
\end{align*}
which implies 
\begin{align*}
	\E[\Phi(\delta_t)-f(\delta_t,\theta_t^K)|\theta_t^{K-1}]\leq \Phi(\delta_t)-f(\delta_t,\theta_t^{K-1})-\frac{1}{L_{22}}\|\nabla_\theta f(\delta_t,\theta^{K-1}_t)\|^2+\frac{1}{2L_{22}}\E[\|\widehat{\nabla}_\theta f(\delta_t,\theta_t^{K-1},\xi_t^{K-1})\|^2|\theta_t^{K-1}].
\end{align*}
 By \Cref{ass:sto}, we have 
 \begin{align*}
 	\E[\|\widehat{\nabla}_\theta f(\delta_t,\theta_t^{K-1},\xi_t^{K-1})\|^2|\theta_t^{K-1}]&=\E[\|\widehat{\nabla}_\theta f(\delta_t,\theta_t^{K-1},\xi_t^{K-1})-\nabla_\theta f(\delta_t,\theta_t)+\nabla_\theta f(\delta_t,\theta_t)\|^2|\theta_t^{K-1}]\\
 	&= \E[\|\widehat{\nabla}_\theta f(\delta_t,\theta_t^{K-1})-\nabla_\theta f(\delta_t,\theta_t)\|^2|\theta_t^{K-1}]+ \E[\|\nabla_\theta f(\delta_t,\theta_t)\|^2|\theta_t^{K-1}]\\
 	&\qquad +2\E[\left\langle\widehat{\nabla}_\theta f(\delta_t,\theta_t^{K-1},\xi_t^{K-1})-\nabla_\theta f(\delta_t,\theta_t), \nabla_\theta f(\delta_t,\theta_t) \right\rangle|\theta_t^{K-1}]\\
 	&\leq 3 \E[\|\widehat{\nabla}_\theta f(\delta_t,\theta_t^{K-1})-\nabla_\theta f(\delta_t,\theta_t)\|^2|\theta_t^{K-1}] + \frac{3}{2}\E[\|\nabla_\theta f(\delta_t,\theta_t)\|^2|\theta_t^{K-1}]\\
 	& \leq \frac{3\sigma^2}{M}+\frac{3}{2}\E[\|\nabla_\theta f(\delta_t,\theta_t)\|^2|\theta_t^{K-1}]
 \end{align*} 
where the second last inequality follows from Young's inequality. Combing the two inequalities above, we arrive at 
\begin{align*}
	\E[\Phi(\delta_t)-f(\delta_t,\theta_t^K)|\theta_t^{K-1}]&\leq \Phi(\delta_t)-f(\delta_t,\theta_t^{K-1})-\frac{1}{4L_{22}}\|\nabla_\theta f(\delta_t,\theta^{K-1}_t)\|^2+\frac{3\sigma^2}{2L_{22}M}\\
	& \leq (1-\frac{2\mu}{L_{22}})(\Phi(\delta_t)-f(\delta_t,\theta_t^{K-1}))+\frac{3\sigma^2}{2L_{22}M}.
\end{align*} 
Recursively applying the above inequality, and let $\rho=1-\frac{2\mu}{L_{22}}<1$, we obtain
\begin{align*}
		\E[\Phi(\delta_t)-f(\delta_t,\theta^K_t)]\leq \rho^K (\Phi(\delta_t)-f(\delta_t,\theta_0))+\left(\frac{1-\rho^K}{1-\rho}\right)\frac{3\sigma^2}{2L_{22}M}
	\end{align*}
	As shown in \citep{gda_two}, there exists a constant $B$, such that $\Phi(\delta_t)-f(\delta_t,\theta_0)\leq B$ holds for every $t$. Then, using the quadratic growth shown in the proof of \Cref{lemma:phi_diff}, we have 
	\begin{align*}
		\E[\|\theta^K_t-\theta^*_t\|^2]\leq \frac{B}{2\mu}\rho^K+ \left(\frac{1-\rho^K}{1-\rho}\right)\frac{3\sigma^2}{2L_{22}M}	
		\end{align*}
		
Hence, with Jensen inequality and the choice of $K$ and $M$
\begin{align*}
	\E[\|z_t\|]&=\E[\|\nabla_\delta f(\delta_t,\theta^K_t)-\nabla \Phi(\delta_t)\|]\\
	&\leq L_{12}\E[\|\theta^K_t-\theta^*_t\|]\\
	&\leq  L_{12}\sqrt{\frac{B}{2\mu}\rho^K+ \left(\frac{1-\rho^K}{1-\rho}\right)\frac{3\sigma^2}{2L_{22}M}}\\
	&\leq \frac{L\epsilon^2}{4D (2\Phi_\infty+{L}D)^2},
\end{align*}
which proves \eqref{eq:zt_ineq}. To prove \eqref{eq:theta_fosp}, note that 
\begin{align*}
	\E[\|\nabla_\theta f(\delta_t,\theta^K_t)\|]&=\E[\|\nabla_\theta f(\delta_t,\theta^K_t)-\nabla_\theta f(\delta_t,\theta^*_t)\|]\\
	&\leq L_{22} \E[\|\theta^K_t-\theta^*_t\|]\leq \frac{L\epsilon^2}{4D (2\Phi_\infty+{L}D)^2}\leq \epsilon.
\end{align*}
\end{proof}

We are now ready to present the full version of \Cref{thm:intermittent} and its proof in the following.
\begin{theorem}[Full version of \Cref{thm:intermittent}]
	Under \Cref{ass:minmaxpl}, \Cref{ass:lip} and \Cref{ass:sto}, for any given $\epsilon\in (0,1)$, let the step sizes be chosen as $\eta_1=1/L_{22}, \eta_2=1/L$, if $K$, $N$ and  $M$ are large enough, 
	\begin{align*}
		N\geq N(\epsilon)\sim \mathcal{O}(\epsilon^{-2}), \qquad K\geq K(\epsilon)\sim \mathcal{O}(\log\epsilon^{-1}),\qquad M\geq M(\epsilon)\sim \mathcal{O}(\epsilon^{-4})
	\end{align*}
	then there exists an index $t$ such that $(\delta_t,\theta^K_t)$ is an $\epsilon$-FOSP.
\end{theorem}

\begin{proof}[Proof of \Cref{thm:intermittent}]
	Due to projection, 
	\begin{align*}
		\left\langle\delta_t-\frac{1}{L}\nabla_\delta f(\delta_t,\theta^K_t,\xi_t^{K})-\delta_{t+1},\delta-\delta_{t+1}\right\rangle\leq 0, \quad \forall \delta\in \Delta.
	\end{align*} 
	Let $\delta=\delta_t$, and take expectations on both sides, we have 
	\begin{align*}
		\E[\left\langle\nabla_\delta f(\delta_t,\theta^K_t), \delta_{t+1}-\delta_t\right\rangle|\delta_t,\theta_t^K]\leq -L\E[\|\delta_t-\delta_{t+1}\|^2|\delta_t,\theta_t^K],
	\end{align*}
	which leads to 
	\begin{align}
	\E[\left\langle\nabla_\delta f(\delta_t,\theta^*_t),\delta_{t+1}-\delta_t\right\rangle|\delta_t,\theta_t^K]
		&\leq \E[\left\langle\nabla_\delta f(\delta_t,\theta^*_t)-\nabla_\delta f(\delta_t,\theta^K_t), \delta_{t+1}-\delta_t\right\rangle|\delta_t,\theta_t^K]-L\E[\|\delta_t-\delta_{t+1}\|^2|\delta_t,\theta_t^K]\nonumber \\
		&=\E[\left\langle z_t, \delta_t-\delta_{t+1}\right\rangle|\delta_t,\theta_t^K]-L\E[\|\delta_t-\delta_{t+1}\|^2|\delta_t,\theta_t^K].\label{eq:zt_leq}
	\end{align}
	
	On the other hand, since $\Phi$ is L-Lipschitz smooth, 
	\begin{align}
		\E[\Phi(\delta_{t+1})]&\leq \E[\Phi(\delta_t)]+\E[\left\langle\nabla_\delta f(\delta_t,\theta^*_t),\delta_{t+1}-\delta_t\right\rangle]+\frac{L}{2}\E[\|\delta_{t+1}-\delta_t\|^2]\nonumber \\
		&\leq \E[\Phi(\delta_t)]+\E[\left\langle z_t, \delta_t-\delta_{t+1}\right\rangle] -\frac{L}{2}\E[\|\delta_{t+1}-\delta_t\|^2],\label{eq:phi_taylor}
	\end{align} 
	where the last inequality holds because of \Cref{eq:zt_leq} and law of total probability.
	
	Also by projection, 
	\begin{align}
		\E[\left\langle\nabla_\delta f(\delta_t, \theta^K_t), \delta-\delta_{t+1}\right\rangle|\delta_t,\theta_t^K] & \geq L\E[\left\langle\delta_t-\delta_{t+1},\delta-\delta_{t+1} \right\rangle|\delta_t,\theta_t^K]\nonumber \\
		&\geq \E[\left\langle \nabla_\delta f(\delta_t,\theta^K_t), \delta_{t+1}-\delta_t\right\rangle|\delta_t,\theta_t^K]+L\E[\left\langle\delta_t-\delta_{t+1},\delta-\delta_{t+1} \right\rangle|\delta_t,\theta_t^K]\nonumber \\
		&\geq \E[\left\langle\nabla_\delta f(\delta_t,\theta^K_t-\Phi(\delta_t), \delta_{t+1}-\delta_t)\right\rangle|\delta_t,\theta_t^k]\nonumber \\
		&\quad + \E[\left\langle\nabla\Phi(\delta_t), \delta_{t+1}-\delta_t\right\rangle|\delta_t,\theta_t^k]+L\E[\left\langle\delta_t-\delta_{t+1},\delta-\delta_{t+1} \right\rangle|\delta_t,\theta_t^k]\nonumber \\
		&\geq -\|z_t\|\E[\|\delta_{t+1}-\delta_t\||\delta_t,\theta_t^k]-\|\nabla\Phi(\delta_t)\|\E[\|\delta_{t+1}-\delta_t\||\delta_t,\theta_t^k]\nonumber\\
		&\qquad-L\E[\|\delta-\delta_{t+1}\|\|\delta_{t+1}-\delta_t\||\delta_t,\theta_t^k]\nonumber \\
		&\geq -(\Phi_\infty+LD+\|z_t\|)\E[\|\delta_{t+1}-\delta_t\||\delta_t,\theta_t^K],\label{eq:et_1}
	\end{align}
	where the last inequality is obtained using our assumptions that $\Phi_\infty=\max\{\max_{\delta}\|\nabla \Phi\|,1\}$, and the diameter of the set $\Delta$ is $D$. 
	
	Notice that by the choice of $K$ in \Cref{lem:zt}, 
	\begin{align*}
		\E[\|z_t\|]&\leq L_{12}\E[\|\theta^K_t-\theta^*_t\|]\\
		&\leq \frac{L\epsilon^2}{4D (2\Phi_\infty+{L}D)^2} \leq \Phi_\infty,
	\end{align*}
	which reduces \eqref{eq:et_1} to the following
	\begin{align*}
		\E[\left\langle\nabla_\delta f(\delta_t, \theta^K_t), \delta-\delta_{t+1}\right\rangle]\geq -(2\Phi_\infty+LD)\E[\|\delta_{t+1}-\delta_t\|].
	\end{align*}
	Let $e_t=-\E[\left\langle\nabla_\delta f(\delta_t, \theta^K_t), \delta-\delta_{t+1}\right\rangle]$, then  
 we have 
 \begin{align}\label{eq:et_frac}
 	\E[\|\delta_{t+1}-\delta_t\|]\geq \frac{e_t}{2\Phi_\infty+LD},
 \end{align}
 and combining \eqref{eq:et_frac} with \eqref{eq:phi_taylor} yields 
 \begin{align*}
 	\E[\Phi(\delta_{t+1})]-\E[\Phi(\delta_t)]\leq -\frac{L}{2}\frac{e_t^2}{(2\Phi_\infty+LD)^2}+D\|z_t\|.
 \end{align*}
 Therefore, we have 
 \begin{align*}
 	\Phi(\delta_0)-\E[\Phi(\delta_N)]\geq -\frac{L}{2(2\Phi_\infty+LD)^2}\sum_{t=0}^{N-1}e_t^2-D\sum_{t=0}^{N-1}\E[\|z_t\|].
 \end{align*}
Denote $D_\Phi:=\Phi(\delta_0)-\min_\delta\Phi(\cdot)$, and the the above inequality can be rewritten as 
\begin{align*}
	\frac{2D_\phi(2\Phi_\infty+LD)^2}{LN}+\frac{2D(2\Phi_\infty+LD)^2}{LN}\sum_{t=0}^{N-1}\E[\|z_t\|].
\end{align*}
Let $N\geq \frac{4D_\Phi(2\Phi_\infty+LD)^2}{L\epsilon^{2}}$, and then by \Cref{lem:zt}, we obtain 
\begin{align*}
	\frac{1}{N}\sum_{t=0}^{N-1}e_t^2\leq \epsilon^2,
\end{align*}
which implies that there exists one index $t$ for which 
$e_t\leq \epsilon$. This completes the proof. 
\end{proof}

\subsection{Proofs in \Cref{sec:analysis_persistent}}\label{app:per_proof}
This section presents the full version of \Cref{thm:persistent} with detailed choices of step sizes $\eta_1,\eta_2$ and batch size $M$, and the main results rest on the following lemmas.   

\begin{lemma}\label{lemma:decrease_phi}
	The iterates $\{\delta_t\}$ generated by \Cref{algo:persist} satisfies the following inequality,
	\begin{align}\label{eq:decreasing_bound}
		\begin{aligned}
	\E[\Phi(\delta_{t+1})]&\leq \E[\Phi(\delta_{t})]- (\frac{\eta_2}{2}-2L\eta_2^2)\E[\|\nabla \Phi(\delta_{t})\|^2]\\
	&\qquad+(\frac{\eta_2}{2}+2L\eta_2^2)\E[\|\nabla \Phi(\delta_{t})-\nabla_{\delta}f(\delta_{t},\theta_{t})\|^2]+ \frac{L\eta_2^2\sigma^2}{M}.
	\end{aligned}
	\end{align}
\end{lemma}
\begin{proof}[Proof of \Cref{lemma:decrease_phi}]
	Since $\Phi(\delta)$ is $L$-smooth (see \Cref{lemma:phi_diff}), we have
	\begin{align*}
			\Phi(\delta_{t+1})-\Phi(\delta_t)-(\delta_{t+1}-\delta_t)\nabla\Phi(\delta_t)\leq L/2\|\delta_{t+1}-\delta_t\|^2.
	\end{align*}   
	In \Cref{algo:persist}, gradient descent is used to update $\delta_t$, hence $\delta_{t+1}-\delta_t=\eta_2 \hat{\nabla}_{\delta}f(\delta_t,\theta_t;\xi_t)$, and the above inequality can be rewritten as 
	\begin{align*}
		\Phi(\delta_{t+1})&\leq \Phi(\delta_t)-\eta_2 \hat{\nabla}_\delta f(\delta_t,\theta_t;\xi_t)^\tp \nabla\Phi(\delta_t) +\frac{L}{2}\|\delta_{t+1}-\delta_t\|^2\\
		&= \Phi(\delta_t)+\eta_2(\nabla\Phi(\delta_t)-\hat{\nabla}_{\delta}f(\delta_t,\theta_t;\xi_t))^\tp\nabla\Phi(\delta_t)\\
		&\qquad-\eta_2\|\nabla\Phi(\delta_t)\|^2+\frac{L\eta_2^2}{2}\|\hat{\nabla}_\delta f(\delta_t,\theta_t,\xi_t)\|^2.
	\end{align*}
	By taking a conditional expectation on both sides, we have 
	\begin{align}\label{eq:cond_ineq}
		\E[\Phi(\delta_{t+1})|\delta_t,\theta_t]&\leq \Phi(\delta_t)+\eta_2(\nabla\Phi(\delta_t)-{\nabla}_{\delta}f(\delta_t,\theta_t))^\tp\nabla\Phi(\delta_t)\nonumber \\
		&\quad-\eta_2\|\nabla\Phi(\delta_t)\|^2+\frac{L\eta_2^2}{2}\E[\|\hat{\nabla}_\delta f(\delta_t,\theta_t,\xi_t)\|^2|\delta_{t},\theta_t ].
	\end{align}
	With Cauchy-Schwartz inequality, 
	\begin{align}
		\|\hat{\nabla}_\delta f(\delta_t,\theta_t,\xi_t)\|^2&\leq 2\|\hat{\nabla}_\delta f(\delta_t,\theta_t,\xi_t)-\nabla f(\delta_t,\theta_t)\|^2+2\|\nabla_\delta f(\delta_t,\theta_t)\|^2\label{eq:sc_1}\\
		\|\nabla_\delta f(\delta_t,\theta_t)\|^2&\leq 2\|\nabla\Phi(\delta_t)-\nabla_\delta f(\delta_t,\theta_t)\|^2+2\|\nabla \Phi(\delta_t)\|^2.\label{eq:sc_2}
	\end{align}

	Applying Young's inequality to the second term of the right-hand side in \eqref{eq:cond_ineq}, we have 
		\begin{align}
			\left(\nabla \Phi(\delta_t)-\nabla_\delta f(\delta_t,\theta_t)\right)^\tp \nabla \Phi(\delta_t)\leq \frac{\|\nabla\Phi(\delta_t)-\nabla_\delta f(\delta_t,\theta_t)\|^2+\|\nabla\Phi(\delta_t)\|^2}{2}.\label{eq:young}
		\end{align}
	Therefore, plugging \eqref{eq:sc_1},\eqref{eq:sc_2} and \eqref{eq:young} into \eqref{eq:cond_ineq} leads to the following                               
	\begin{align*}
		&\E[\Phi(\delta_{t+1})|\delta_t,\theta_t]\\
		&\qquad \leq  \Phi(\delta_t)+\frac{\eta_2}{2}\|\nabla\Phi(\delta_t)-\nabla_{\delta}f(\delta_t,\theta_t)\|^2-\frac{\eta_2}{2}\|\nabla\Phi(\delta_t)\|^2\\
		&\qquad+L\eta_2^2\E[\|\hat{\nabla}_\delta f(\delta_t,\theta_t;\xi_t)-\nabla_\delta f(\delta_t,\theta_t)\|^2|\delta_t,\theta_t]+L\eta_2^2 \|\nabla_\delta f(\delta_t,\theta_t)\|^2 \\
		 &\qquad \leq \Phi(\delta_t) +(\frac{\eta_2}{2}+2L\eta_2^2)\|\nabla\Phi(\delta_t)-\nabla_{\delta}f(\delta_t,\theta_t)\|^2\\
		 &\qquad -(\frac{\eta_2}{2}-2L\eta_2^2)\|\nabla\Phi(\delta_t)\|^2+ L\eta_2^2\E[\|\hat{\nabla}_\delta f(\delta_t,\theta_t;\xi_t)-\nabla_\delta f(\delta_t,\theta_t)\|^2|\delta_t,\theta_t].
	\end{align*}
	Finally, with \Cref{ass:sto} and law of total probability, we obtain the desired inequality
\end{proof}

\begin{lemma}\label{lemma:decreasing_error}
	Under \Cref{ass:rsi}, \Cref{ass:lip} and \Cref{ass:sto}, for the iterates generated by \Cref{algo:persist}, let $d_t:=\E[\|\theta^*_t-\theta_t\|^2]$, where $\theta^*_t\in \argmax_{\theta} f(\delta_t,\theta)$, and let $\kappa=2\bar{L}/\mu^2$, where $\bar{L}$ is defined in \Cref{lem:zt}, then 
	\begin{align}
		\begin{aligned}
		d_{t}\leq &\left(1-\frac{1}{2\kappa}+16\kappa^3\bar{L}^2\eta_2^2\right) d_{t-1}+ 16\kappa^3\eta_2^2\E[\|\nabla \Phi(\delta_{t-1})\|^2]+ \frac{8\kappa^3\eta_2^2\sigma^2}{M}+\frac{\mu^2\sigma^2}{L_{22}^2 M}
	\end{aligned}\label{eq:dt_contract}
	\end{align}
\end{lemma}
\begin{proof}[Proof of \Cref{lemma:decreasing_error}]
	According to the gradient update with respect to $\theta$, 
	\begin{align*}
		\left\langle\theta_{t-1}+\eta_1\nabla_\theta f(\delta_{t-1},\theta_{t-1}, \xi_{t-1})-\theta_t,\theta-\theta_t\right\rangle\leq 0, \forall \theta.
	\end{align*}
	Let $\theta=\theta^{*}_{t-1}$, then the above inequality leads to 
	\begin{align}
		&\left\langle\theta_t-\theta_{t-1}-\eta_1\nabla_\theta f(\delta_{t-1},\theta_{t-1}, \xi_{t-1}),\theta^*_{t-1}-\theta_t\right\rangle \geq 0\nonumber\\
		\Leftrightarrow &\underbrace{\left\langle\theta^*_{t-1}-\theta_{t-1}-\eta_1\nabla_\theta f(\delta_{t-1},\theta_{t-1}, \xi_{t-1}), \theta^*_{t-1}-\theta_t\right\rangle}_{(*)} \geq \|\theta^*_{t-1}-\theta_t\|^2.\label{eq:star_ineq}
	\end{align}
	Denote the right-hand side of the above inequality by $(*)$, then 
\begin{align}
	(*)&=\left\langle\theta^*_{t-1}-\theta_{t-1}-\eta_1\nabla_\theta f(\delta_{t-1},\theta_{t-1}, \xi_{t-1}), \theta^*_{t-1}-\theta_{t-1}-\eta_1\nabla_\theta f(\delta_{t-1},\theta_{t-1}, \xi_{t-1})\right\rangle\nonumber\\
	&=\|\theta^*_{t-1}-\theta_{t-1}\|^2-2\eta_1\left\langle\theta^*_{t-1}-\theta_{t-1},\nabla_\theta f(\delta_{t-1},\theta_{t-1},\xi_{t-1})\right\rangle+\eta_1^2\|\nabla\theta f(\delta_{t-1},\theta_{t-1},\xi_{t-1})\|^2.\label{eq:star}
\end{align}
By \Cref{ass:rsi}, we have 
\begin{align}
	\left\langle\nabla_\theta f(\delta_{t-1},\theta_{t-1}), \theta^*_{t-1}-\theta_{t-1} \right\rangle \geq \mu \|\theta^*_{t-1}-\theta_{t-1}\|^2.\label{eq:rsi_app}
\end{align}
Taking expectations on both sides of \eqref{eq:star} and plugging \eqref{eq:rsi_app} yields 
\begin{align}
	\E[(*)]\leq (1-2\mu\eta_1)\E[\|\theta^*_{t-1}-\theta_{t-1}\|^2]+\eta_1^2\E[\|\nabla_\theta f(\delta_{t-1},\theta_{t-1},\xi_{t-1})\|^2].\label{eq:e_star}
\end{align} 
Similar to the argument in the proof of \Cref{lem:zt}, by Young's inequality, \Cref{ass:sto} and \Cref{ass:lip},  
\begin{align*}
	\E[\|\nabla_\theta f(\delta_{t-1},\theta_{t-1},\xi_{t-1})\|^2]&\leq \frac{2\sigma^2}{M}+2\|\nabla_\theta f(\delta_{t-1},\theta_{t-1})\|^2\\
	& \leq \frac{2\sigma^2}{M} +2L_{22}\|\theta^*_{t-1}-\theta_{t-1}\|^2.
\end{align*}
Hence, plug the above inequality and \eqref{eq:e_star} into \eqref{eq:star_ineq} after taking expectations, we obtain 
 \begin{align*}
 	\E[\|\theta_{t-1}^*-\theta_t\|^2]\leq (1-2\mu\eta_1+2L_{22}\eta_1^2)d_{t-1}+\frac{2\sigma^2\eta_1^2}{M}.
 \end{align*}
 Let $\kappa=\frac{2\bar{L}}{\mu^2}$ and with the choice of $\eta_1=\frac{\mu}{2L_{22}}$, the above inequality can be rewritten as
 \begin{align}
 	\E[\|\theta_{t-1}^*-\theta_t\|^2]\leq (1-\frac{1}{\kappa})d_{t-1}+\frac{\sigma^2\mu^2}{2L_{22}^2M}.\label{eq:1-kappa}
 \end{align}
By Young's inequality
 \begin{align*}
 	d_t&=\E[\|\theta_t^*-\theta_t\|^2]=\E[\|\theta_t^*-\theta^*_{t-1}+\theta^*_{t-1}-\theta_t\|^2]\\
 	&\leq \left(1+\frac{1}{2(\max\{\kappa,2\}-1)}\right)\E[\|\theta^*_{t-1}-\theta_t\|^2]+(1+2(\max\{\kappa,2\}-1))\E[\|\theta^*_t-\theta^*_{t-1}\|^2]\\
 	&\leq \left(\frac{2\max\{\kappa,2\}-1}{2\max\{\kappa,2\}-2}\right)\E[\|\theta^*_{t-1}-\theta_t\|^2]+4\kappa \E[\|\theta^*_t-\theta^*_{t-1}\|^2]\\
 	&\leq \left(1-\frac{1}{2\kappa}\right)d_{t-1}+4\kappa \E[\|\theta^*_t-\theta^*_{t-1}\|^2]+\frac{\sigma^2\mu^2}{L_{22}^2M}.
 \end{align*}
By the Lipschitz continuity of $\theta^*$ shown in \eqref{eq:theta_delta}, $\|\theta_t^*-\theta_{t-1}^*\|\leq \frac{L_{12}}{2\mu}\|\delta_t-\delta_{t-1}\|\leq \kappa \|\delta_t-\delta_{t-1}\|$. Without loss of generality, it is  assumed that $\mu$ are small enough. Furthermore, according to the stochastic gradient ascent, we have 
\begin{align*}
	\E[\|\delta_t-\delta_{t-1}\|^2]&=\eta_2^2\E[\|\nabla\delta f(\delta_{t-1},\theta_{t-1},\xi_{t-1})\|^2]\\
	&\leq 4\bar{L}^2\eta_2^2 d_{t-1}+4\eta_2^2 \E[\|\nabla \Phi(\delta_{t-1})\|^2]+\frac{2\sigma^2\eta_2^2}{M}.
\end{align*}
 Putting these pieces together leads to \eqref{eq:dt_contract} in \Cref{lemma:decreasing_error}
\end{proof}

\begin{lemma}\label{lemma:decrease_d}
	Under \Cref{ass:rsi}, \Cref{ass:lip} and \Cref{ass:sto}, for the iterates generated by \Cref{algo:persist}, and for $d_t$ defined in \Cref{lemma:decreasing_error}, 
	\begin{align}
		\E[\Phi(\delta_{t+1})]\leq \E[\Phi(\delta_t)] -\frac{7\eta_2}{16}\E[\|\nabla \Phi(\delta_t)\|^2] + \frac{9\bar{L}^2 \eta_2 d_{t}}{16}+\frac{L\eta_2^2\sigma^2}{M}.\label{eq:control}	\end{align} 
\end{lemma}

\begin{proof}[Proof of \Cref{lemma:decrease_d}]
	By the choice of $\eta_2=\min\{1/32L, 1/8\kappa^2\bar{L}\}$, 
	\begin{align*}
		\frac{7}{16}\eta_2\leq \frac{\eta_2}{2}-2L\eta_2^2\leq \frac{\eta_2}{2}+2L\eta_2^2\leq \frac{9}{16}\eta_2,
	\end{align*}
	which reduces \eqref{eq:decreasing_bound} to 
	\begin{align*}
			\E[\Phi(\delta_{t+1})]&\leq \E[\Phi(\delta_{t})]- \frac{7\eta_2}{16}\E[\|\nabla \Phi(\delta_{t})\|^2]\\
	&\qquad+\frac{9\eta_2}{16}\E[\|\nabla \Phi(\delta_{t})-\nabla_{\delta}f(\delta_{t},\theta_{t})\|^2]+ \frac{L\eta_2^2\sigma^2}{M}.
	\end{align*}
	Note that
	\begin{align*}
		\E[\|\nabla \Phi(\delta_{t})-\nabla_{\delta}f(\delta_{t},\theta_{t})\|^2]\leq \bar{L}^2 \E[\|\theta_t^*-\theta_t\|^2]=\bar{L}^2 d_{t},
	\end{align*} 
	Combining the above inequalities, we arrive at
	\eqref{eq:control}.
\end{proof}
\Cref{lemma:decrease_d} implies that the difference $\E[\Phi(\delta_N)]-\E[\Phi(\delta_0)]$ is controlled by the sum of $d_t$, that is,
\begin{align*}
	\E[\Phi(\delta_N)]-\E[\Phi(\delta_0)]\leq -\mathcal{O}(\eta_2) \left(\sum_{t=0}^{N-1}\E[\|\nabla\Phi(\delta_t)\|^2]\right)+\mathcal{O}(\eta_2)\sum_{t=0}^{N-1}d_t.
\end{align*}
Since $d_t$ exhibits a linear contraction as shown in \Cref{lemma:decreasing_error}, $\sum_{t=0}^{N-1}d_t$ can be further controlled by $\sum_{t=0}^{N-1}\E[\|\nabla\Phi(\delta_t)\|^2]$. Eventually, one can show that $\sum_{t=0}^{N-1}\E[\|\nabla\Phi(\delta_t)\|^2]$ is bounded, implying the convergence of \Cref{algo:persist}. The full version of \Cref{thm:persistent} is stated as follows.

\begin{theorem}
Under \Cref{ass:rsi}, \Cref{ass:lip} and \Cref{ass:sto}, for any given $\epsilon\in (0,1)$, let the step sizes be chosen as $\eta_1=\mu/L_{22}$, $\eta_2=\min\{1/32L, 1/8\kappa^2\bar{L}\}$,   then if $N$ and $M$ are large enough, 
\begin{align*}
	N\geq \left(\frac{(8\kappa^2+32)\bar{L}D_{\Phi}+\kappa D^2\bar{L}^2}{\epsilon^2}\right),\qquad M\geq \mathcal{O}(\frac{\sigma^2\kappa }{\epsilon^{2}}),
\end{align*} 	
there exists an index $t$, such that $(\delta_t,\theta_t)$ is an $\epsilon$-FOSP.  
\end{theorem}
\begin{proof}[Proof of \Cref{thm:persistent}]
Let $\gamma=1-\frac{1}{2\kappa}+16\kappa^3\bar{L}^2\eta_2^2$. Applying \eqref{eq:dt_contract} recursively yields 
\begin{align}
	d_t\leq \gamma^t D^2+16\kappa^3\eta_2^2 \left(\sum_{j=0}^{t-1}\E[\|\nabla\Phi(\delta_j)\|^2]\right)+\left(\frac{8\kappa^3\eta_2^2\sigma^2}{M}+\frac{\mu^2\sigma^2}{L_{22}^2 M}\right)\left(\sum_{j=0}^{t-1} \gamma^{t-1-j}\right).\label{eq:dt_recur}
\end{align}
Combining \eqref{eq:control} and \eqref{eq:dt_recur}, we obtain
\begin{align}
	\E[\Phi(\delta_t)]&\leq \E[\Phi(\delta_{t-1})]-\frac{7\eta_2}{16}\E[\|\nabla\Phi(\delta_{t-1})\|^2]+\frac{9\bar{L}^2\eta_2}{16}\gamma^{t-1} D^2\nonumber \\
	&\qquad+ 9\bar{L}^2\kappa^3\eta^3_2\left(\sum_{j=0}^{t-2}\gamma^{t-2-j}\E[\|\nabla\Phi(\delta_j)\|^2]\right)+\frac{9\bar{L}^2\eta_2}{16}\left(\frac{8\kappa^3\eta_2^2\sigma^2}{M}+\frac{\mu^2\sigma^2}{L_{22}^2 M}\right)\left(\sum_{j=0}^{t-2} \gamma^{t-2-j}\right)\nonumber\\
	&\qquad + \frac{L\eta_2^2\sigma^2}{M}.\label{eq:pre_recur}
\end{align}
Applying \eqref{eq:pre_recur} recursively leads to 
\begin{align*}
	\E[\Phi(\delta_{N})]&\leq \Phi(\delta_0)-\frac{7\eta_2}{16}\sum_{t=0}^{N-1}\E[\|\nabla\Phi(\delta_t)\|^2]+\frac{9\bar{L}^2\eta_2D^2}{16} \left(\sum_{t=0}^{N-1}\gamma^t\right)\\
	&\qquad + \frac{NL\eta_2^2\sigma^2}{M}+9\bar{L}^2\kappa^3\eta^3_2\left(\sum_{t=1}^{N-1}\gamma^{t-2-j}\sum_{j=0}^{t-2}\E[\|\nabla\Phi(\delta_j)\|^2]\right)\\
	&\qquad + \frac{9\bar{L}^2\eta_2}{16}\left(\frac{8\kappa^3\eta_2^2\sigma^2}{M}+\frac{\mu^2\sigma^2}{L_{22}^2 M}\right)\left(\sum_{t=1}^{N}\sum_{j=0}^{t-2} \gamma^{t-2-j}\right).
\end{align*}
By the choice of $\eta_2$, $\gamma\leq 1-\frac{1}{4\kappa}$, and hence $\sum_{t=0}^{N-1}\gamma^t\leq 4\kappa$, which implies 
\begin{align*}
	\sum_{t=1}^{N-1}\gamma^{t-2-j}\sum_{j=0}^{t-2}\E[\|\nabla\Phi(\delta_j)\|^2]&\leq 4\kappa \sum_{t=0}^{N-1}\E[\|\nabla\Phi(\delta_t)\|^2],\\
	\sum_{t=1}^{N}\sum_{j=0}^{t-2} \gamma^{t-2-j}\leq 4\kappa N.
\end{align*}
	Combining all the inequalities above, and suppressing the constants in front of variance terms, we arrive at 
	\begin{align*}
		\E[\Phi(\delta_{N})]&\leq \Phi(\delta_0)-\frac{47\eta_2}{128}\sum_{t=0}^{N-1}\E[\|\nabla\Phi(\delta_t)\|^2]+\mathcal{O}(\frac{\eta_2\sigma^2\kappa N}{M})+\mathcal{O}({\kappa\eta_2D^2\bar{L}^2}).
	\end{align*}
	By the definition of $D_\Phi$, the above inequality leads to 
	\begin{align*}
		\frac{1}{N}\sum_{t=0}^{N-1}\E[\|\nabla\Phi(\delta_t)\|^2]&\leq \frac{128D_{\Phi}}{47\eta_2 N}+\mathcal{O}(\frac{\kappa D^2\bar{L}^2}{N})+\mathcal{O}(\frac{\sigma^2\kappa }{M})\\
		&\leq \mathcal{O}(\frac{(8\kappa^2+32)\bar{L}D_{\Phi}+\kappa D^2\bar{L}^2}{N})+\mathcal{O}(\frac{\sigma^2\kappa}{M}).
	\end{align*}
	Since $N\geq\mathcal{O}\left(\frac{(8\kappa^2+32)\bar{L}D_{\Phi}+\kappa D^2\bar{L}^2}{\epsilon^2}\right) $, and $M\geq \mathcal{O}(\frac{\sigma^2\kappa }{\epsilon^{2}})$,
	\begin{align*}
		\frac{1}{N}\sum_{t=0}^{N-1}\E[\|\nabla\Phi(\delta_t)\|^2]\leq \epsilon^2.
	\end{align*}
	Following the same argument in the proof of \Cref{thm:intermittent}, we arrive at the conclusion.
\end{proof}
\end{document}